\documentclass[twocolumn, journal]{IEEEtran}
\usepackage{graphicx}
\usepackage{epstopdf}
\usepackage{amsthm}
\theoremstyle{definition}
\theoremstyle{assumption}
\usepackage{epsfig}
\usepackage{latexsym}
\usepackage{amsfonts}
\usepackage{here}
\usepackage{rawfonts}
\usepackage[T1]{fontenc}
\usepackage{calc}
\usepackage{capitalgreekitalic}
\usepackage{url}
\usepackage{enumerate}
\usepackage{color}
\usepackage[tbtags]{amsmath}
\usepackage{amssymb}
\usepackage{upref}
\usepackage{epic,eepic}
\usepackage{times}
\usepackage{dsfont}
\usepackage{comment}
\usepackage{cite}
\usepackage{bbm}
\usepackage{bm}
\usepackage{amsmath}
\usepackage{array}
\usepackage{tabularx}
\usepackage{threeparttable}
\usepackage{fancyhdr}

\allowdisplaybreaks[4]

\newcommand{\diag}{\ensuremath{\operatorname{diag}}}

\newtheorem{theorem}{\bf Theorem}

\newtheorem{lemma}{\bf Lemma}

\newtheorem{definition}{\bf Definition}
\newtheorem{remark}{\bf Remark}

\usepackage{dsfont}

\newcounter{step}

\newlength{\totlinewidth}
\newenvironment{algorithm}{%
  \rule{\linewidth}{1pt}
  \begin{list}{}%
    {\usecounter{step}%
      \settowidth{\labelwidth}{\textbf{Step 2:}}%
      \setlength{\leftmargin}{\labelwidth}%
      \setlength{\topsep}{-2pt}%
      \addtolength{\leftmargin}{\labelsep}%
      \addtolength{\leftmargin}{2mm}%
      \setlength{\rightmargin}{2mm}%
      \setlength{\totlinewidth}{\linewidth}%
      \addtolength{\totlinewidth}{\leftmargin}%
      \addtolength{\totlinewidth}{\rightmargin}%
      \setlength{\parsep}{0mm}%
      \raggedright}}%
  {\end{list}%
  \rule{\linewidth}{1pt}}
\newcounter{substep}

  {\end{list}}

\newlength{\aligntop}
\newlength{\alignbot}
 \makeatletter
\renewenvironment{align}{%
  \vspace{\aligntop}
  \start@align\@ne\st@rredfalse\m@ne
}{%
  \math@cr \black@\totwidth@
  \egroup
  \ifingather@
    \restorealignstate@
    \egroup
    \nonumber
    \ifnum0=`{\fi\iffalse}\fi
  \else
    $$%
  \fi
  \ignorespacesafterend%
  \vspace{\alignbot}\par\noindent
} \makeatother

\IEEEoverridecommandlockouts

\usepackage{algorithm}
\usepackage{algorithmic}

\usepackage{subfigure}

\def\L{{\cal L}}

\begin{document}

\title{A Secure and Private Distributed \\ Bayesian Federated Learning Design}

\author{{Nuocheng Yang,} \emph{Student Member, IEEE}, 
{Sihua Wang, Zhaohui Yang, Mingzhe Chen} \emph{Member, IEEE}, \\
{Changchuan Yin}, \emph{Senior Member, IEEE}, and
{Kaibin Huang}, \emph{Fellow, IEEE}.
\thanks{N. Yang, S. Wang, and C. Yin are with the Beijing Laboratory of Advanced Information Network, and the Beijing Key Laboratory of Network System Architecture and Convergence, Beijing University of Posts and Telecommunications, Beijing 100876, China (e-mails: \protect\url{{yangnuocheng, sihuawang, ccyin}@bupt.edu.cn)}.}
\thanks{Z. Yang is with the College of Information Science and Electronic Engineering, Zhejiang University, Hangzhou 310027, China, and Zhejiang Provincial Key Lab of Information Processing, Communication and Networking (IPCAN), Hangzhou 310007, China (e-mail: \protect\url{yang_zhaohui@zju.edu.cn)}.}
\thanks{M. Chen is with the Department of Electrical and Computer Engineering and Institute for Data Science and Computing, University of Miami, Coral Gables, FL, 33146 USA (email: \protect\url{mingzhe.chen@miami.edu}).}
\thanks{K. Huang is with the Department of Electrical and Electronic Engineering, The University of Hong Kong, Hong Kong (e-mail: \protect\url{huangkb@hku.hk)}.}
}

\maketitle

\begin{abstract}
Distributed Federated Learning (DFL) enables decentralized model training across large-scale systems without a central parameter server. However, DFL faces three critical challenges: privacy leakage from honest-but-curious neighbors, slow convergence due to the lack of central coordination, and vulnerability to Byzantine adversaries aiming to degrade model accuracy. To address these issues, we propose a novel DFL framework that integrates Byzantine robustness, privacy preservation, and convergence acceleration. Within this framework, each device trains a local model using a Bayesian approach and independently selects an optimal subset of neighbors for posterior exchange. We formulate this neighbor selection as an optimization problem to minimize the global loss function under security and privacy constraints. Solving this problem is challenging because devices only possess partial network information, and the complex coupling between topology, security, and convergence remains unclear. To bridge this gap, we first analytically characterize the trade-offs between dynamic connectivity, Byzantine detection, privacy levels, and convergence speed. Leveraging these insights, we develop a fully distributed Graph Neural Network (GNN)-based Reinforcement Learning (RL) algorithm. This approach enables devices to make autonomous connection decisions based on local observations. Simulation results demonstrate that our method achieves superior robustness and efficiency with significantly lower overhead compared to traditional security and privacy schemes.
\end{abstract}


\IEEEpeerreviewmaketitle
\section{Introduction}
Federated learning (FL) allows wireless devices to cooperatively train a machine learning (ML) model without local data transmission, which ensures data locally and reduces communication overhead \cite{FLjsac, AJointLearningandCommunicationsFrameworkforFederatedLearningOverWirelessNetworks, RobustFederatedLearning, MaximizingUncertaintyforFederatedLearningviaBayesianOptimizationBasedModelPoisoning}. 
Existing FL frameworks can generally be categorized into two main paradigms: centralized FL (CFL) \cite{M1} and distributed FL (DFL) \cite{CGHS}.
In CFL, each device needs to upload its local model to a parameter server for global model aggregation, which is not always feasible due to limited resources (i.e., energy and bandwidth) in practical wireless networks \cite{M2}. 
On the other hand, DFL enables devices to collaboratively train an ML model by exchanging local FL models with a subset of neighbors through spontaneous device-to-device (D2D) connections, making it suitable for deployment in large-scale systems. 
Moreover, frequent model transmission through D2D connections is also vulnerable to security issues \cite{ByzantineAttack, TrustworthyfederatedlearningSurvey, DataandModelPoisoningBackdoorAttacks}, caused by poisoned model updates from Byzantine adversaries, and privacy issues \cite{DLG, BatchDLG, AIJack}, caused by private data reconstruction.

Recently, a number of existing works such as \cite{TheLimitationsofFederatedLearninginSybilSettings, CONTRA, ToNSec, Ghavamipour2024PrivacyPreservingAF, Decentralizedfederatedlearningthroughproxymodelsharing, SMC, TowardsEfficientandPrivacyPreservingFederatedDeepLearning, ToNDP, ToNHEzkP, FederatedLearningWithDifferentialPrivacy} have studied the security and privacy issues in CFL framework.
For security issues, the authors in \cite{TheLimitationsofFederatedLearninginSybilSettings, CONTRA, Ghavamipour2024PrivacyPreservingAF, ToNSec} detected Byzantine adversaries by comparing Euclidean distance and cosine similarity between poisoned and normal model updates.
Moreover, for privacy issues, the authors in \cite{Decentralizedfederatedlearningthroughproxymodelsharing, SMC, TowardsEfficientandPrivacyPreservingFederatedDeepLearning, ToNDP, ToNHEzkP, FederatedLearningWithDifferentialPrivacy} utilized differential privacy, secure multiparty computation, and homomorphic encryption to ensure that the model updates are invisible to data reconstruction adversaries.
However, jointly considering security and privacy issues is challenging since security-ensuring schemes need accurate model updates for comparison, while privacy-preserving schemes need invisible model updates for confidentiality.
To tackle this problem, the authors in \cite{FLOD, SplitAggregationLightweightPrivacyPreservingFederatedLearning} utilized multiple servers to collect different parts of accurate model updates from all devices while no single server can obtain complete model updates, thus ensuring that model updates are invisible.
However, the methods in \cite{FLOD, SplitAggregationLightweightPrivacyPreservingFederatedLearning} relied on additional centralized servers which are unsuitable for the DFL framework. 
In addition, the authors in \cite{PrivacyEnhancedFederatedLearningAgainstPoisoningAdversaries, PrivacyPreservingByzantineRobustFederatedLearningviaBlockchainSystems} utilized the clustering method to distinguish Byzantine adversaries from normal devices by comparing the distance between their encrypted model updates.
Obviously, only the devices with independent and identically distributed (IID) training datasets can generate the same encrypted model update, which is unrealizable in non-IID scenarios.

Some related works \cite{Trustiness_basedhierarchicaldecentralizedfederatedlearning, Enhancingprivacypreservation} jointly considered the security and privacy issues in the DFL framework.
The authors in \cite{Trustiness_basedhierarchicaldecentralizedfederatedlearning} proposed a hierarchical DFL framework that assigns the devices to different clusters where the devices in each cluster need to transmit its model updates encrypted by CheonKim-Kim-Song (CKKS) fully homomorphic encryption technique to the same device (referred as parameter server) for intra-cluster model aggregation, thus ensuring privacy preserving.
Then, the parameter servers exchange aggregated model updates after decryption with neighboring parameter servers for security inter-cluster model aggregation, and the model updates with poor validation performance would be removed.
However, the parameter server in each cluster suffers from heavy computational and communication overhead, thus causing unfairness and instability in the DFL framework.
Moreover, the parameter server cannot detect the poisoned updates as the model updates are encrypted during the intra-cluster model aggregation, which threatens security.
In addition, DFL would collapse if poisoned updates accounted for the majority as in intra-cluster the models are average aggregated.
Different from \cite{Trustiness_basedhierarchicaldecentralizedfederatedlearning} where the devices transmit their encrypted model updates to the parameter server within the cluster, the devices in \cite{Enhancingprivacypreservation} exchange the encrypted model updates with all available neighbors and share security proof generated by zero-knowledge succinct non-interactive argument of knowledge (zk-SNARK) method to validate the security of the exchanged model updates.
However, generating security proof leads to substantial computational and storage overheads \cite{zkBench}.
Moreover, these prior works \cite{Trustiness_basedhierarchicaldecentralizedfederatedlearning, Enhancingprivacypreservation} suffer from heavy cryptographic computation overhead while ignoring the acceleration of DFL learning and the device in DFL can only obtain partial information (e.g., device connection and model updates) from its neighbors.

To tackle encryption burden, privacy-security conflict, and slow convergence rate, we propose a novel lightweight DFL framework to achieve Byzantine robustness, privacy preserving, and convergence acceleration. 
In particular, the proposed DFL framework enables each device first to train the local model using the Bayesian approach and then independently select a subset of its neighbors to exchange model updates.
Compared to the traditional ML methods that treat model weights as deterministic variables, the Bayesian approach estimates the distribution of model weights by a Maximum a Posteriori (MAP) estimator based on the local prior and training dataset.
In this way, the devices only need to exchange the updated model's posterior with neighbors and integrate the received distribution into the local prior.
Our key contributions are as follows:

\begin{itemize}
    \item{
    We propose a novel DFL framework in which distributed devices independently select a subset of neighbors (i.e., device connection scheme) to exchange the updated posterior based on partial information (e.g., device connection and model updates).
    Particularly, to detect Byzantine adversaries, the proposed framework enables each device to compare the received models with the expected result formulated from the shared model updates in the last round.
    Moreover, to achieve privacy preserving, the proposed framework enables each device to hide critical information from its neighbors through careful device connection design.
    Compared to the traditional algorithms that detect Byzantine adversaries by comparing the differences between exchanged model updates directly or ensuring data privacy by encryption, the proposed framework offers a distinct perspective from the device connection scheme design which controls the exchange of model updates.
    }
    \item{
    We formulate the selection of neighboring devices as an optimization problem whose goal is to minimize DFL training loss while accounting for privacy, security, and recourse constraints.
    To solve this problem, we first analyze how the device connection scheme impacts the accuracy of the local prior approximation (i.e. expected result).
    This, in turn, affects the private data leakage, and the detection of the Byzantine adversaries.
    Then, we analytically characterize the effect of the device connection scheme on the DFL convergence speed.
    Given these analyses, we can 
    design the device connection based on the characteristics of the device connection matrix incorporating privacy, security, and convergence acceleration requirements.
    }
    \item{
    To learn the optimal device connection based on partial information, we construct a novel distributed graph neural network (GNNs) based reinforcement learning (RL) method.
    By employing GNN-based RL method, each device can adapt to the dynamic dimension of its neighbors and generate an optimal device connection scheme independently.
    Numerical evaluation results show that our proposed DFL framework can achieve security, privacy, and convergence acceleration with lightweight and fair overhead compared to the baselines.
    }
\end{itemize}

\textit{Notations:} Unless otherwise indicated, matrices are represented by bold capital letters (i.e. $\bf{A}$), vectors are denoted by bold lowercase letters (i.e. $\boldsymbol{v}$), and scalars are donated by plain font (i.e. $d$). 
The term $||\boldsymbol{w}||$ donates the L2-norm and $|\boldsymbol{w}|$ represents the L1-norm of $\boldsymbol{w}$.

\section{System Model and Problem Formulation}
\begin{figure}[t]
\centering
\includegraphics[width=8cm]{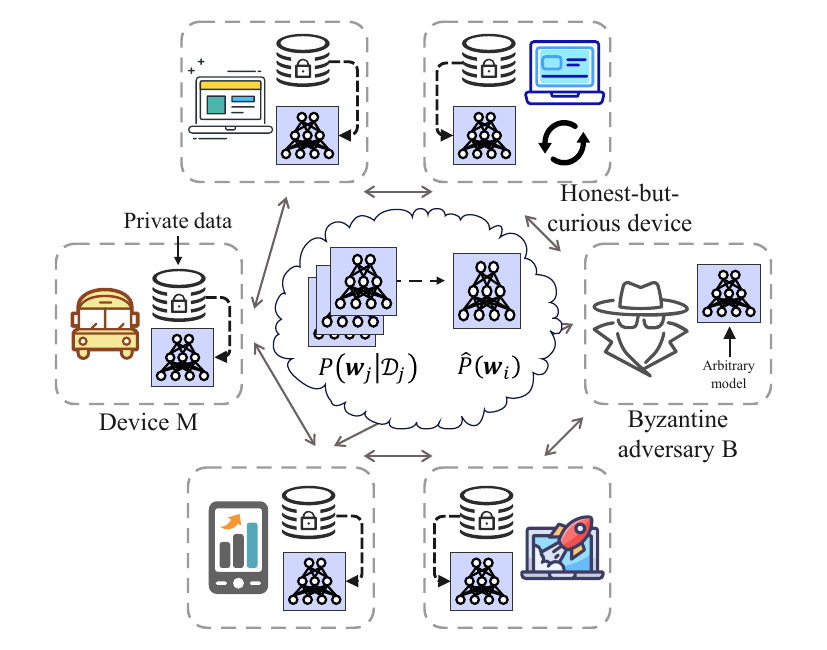}
\vspace{-0.2cm}
\caption{Illustration of the considered DFL model.}
\centering
\vspace{-0.3cm}
\label{fig1}
\end{figure}
Consider a wireless network that consists of a set $\mathcal{R}$ of $R$ mobile devices which can be divided into three categories:
\begin{enumerate}[1)]
    \item \textit{Conventional} device set $\mathcal{M}_1$ with $M_1$ devices first train the local model using the local dataset and then exchange the model updates with a subset of neighbors in $\mathcal{R}$;
    \item \textit{Honest-but-curious} (\textit{HBC}) device set $\mathcal{M}_2$ with $M_2$ devices not only train and exchange model updates (as the conventional devices do) but also attempt to reconstruct private data from other devices; 
    \item \textit{Byzantine adversary} device set $\mathcal{B}$ with $B$ devices transmit randomized model to other devices in $\mathcal{R}$ without local training.
\end{enumerate}

Let $\mathcal{M}=\mathcal{M}_1\cup\mathcal{M}_2$ with $M$ devices be the standard training device set which satisfies $\mathcal{B} \cap \mathcal{M}= \emptyset$ and $\mathcal{B} \cup \mathcal{M} = \mathcal{R}$.

We assume that each device $i$ in $\mathcal{M}$ has a local dataset $\mathcal{D}_i$, containing $N_i$ training data samples with $N=\sum_{i = 1}^{M} N_i$ being the total number of training data samples across devices. 
Each training data sample $n$ consists of an input feature vector $\boldsymbol{x}_{i,n} \in \mathbb{R}^{N_{\textrm{I}}\times 1}$ and a corresponding label vector $\boldsymbol{y}_{i,n} \in \mathbb{R}^{N_{\textrm{O}}\times 1}$.
Then, the DFL problem can be expressed as
\begin{equation}\label{eq:FLloss}
    \mathop {\min }\limits_{\boldsymbol{w}_{1},\cdots,\boldsymbol{w}_{M}} \frac{1}{M} \sum\limits_{i=1}^{M} F\left(\boldsymbol{w}_{i}\right),
    \end{equation}
    \vspace{-0.2cm}
    \begin{align}\label{st1}
    \setlength{\abovedisplayskip}{-40 pt}
    {\rm{s.t.}} \,\,  \boldsymbol{w}_{i} = \boldsymbol{w}_{j}, \forall i,j \in \mathcal{M},\tag{\theequation a}
\end{align} 
where $\boldsymbol{w}_{i} \in \mathbb{R}^{V \times 1}$ is a vector to capture the local FL model on device $i$.
$F\left(\boldsymbol{w}_{i}\right)$ is given by
\begin{equation}\label{eq:LocalLoss}
    F\left(\boldsymbol{w}_{i}\right)=\frac{1}{N_i}\sum\limits_{n=1}^{N_i}f\left(\boldsymbol{w}_{i};\phi\left(\boldsymbol{w}_{i}, \boldsymbol{x}_{i,n}\right),\boldsymbol{y}_{i,n}\right),
\end{equation}
where $\phi\left(\boldsymbol{w}_{i}, \boldsymbol{x}_{i,n}\right)$ denotes the neural network function and $f\left(\boldsymbol{w}_{i};\phi\left(\boldsymbol{w}_{i}, \boldsymbol{x}_{i,n}\right),\boldsymbol{y}_{i,n}\right)$ is the loss function to measure the difference between the output of $\phi\left(\boldsymbol{w}_{i}, \boldsymbol{x}_{i,n}\right)$ and the label $\boldsymbol{y}_{i,n}$. 

\subsection{Training Local Model with Bayesian Approach}
In contrast to the deterministic approaches that treat model weights as unknown deterministic variables, Bayesian learning treats model weights as uncertain distributions.
Bayesian learning estimates the local model by employing the MAP estimator based on the local dataset which is expressed as
\begin{equation}\label{eq:Bayesian_object}
    \begin{aligned}
    \arg\max\limits_{\boldsymbol{w}_{i}} P\left(\boldsymbol{w}_{i}| {\mathcal{D}}_i\right).
    \end{aligned}
\end{equation}
Based on the Bayesian theorem, the posterior $P \left(\boldsymbol{w}_{i}|{\mathcal{D}}_i\right)$ is achieved by
\begin{equation}\label{eq:BayesFunction}
    \begin{aligned}
         P \left(\boldsymbol{w}_{i}| {\mathcal{D}}_i\right)=\frac{P_{\text{l}}\left(\mathcal{D}_i|\boldsymbol{w}_{i}\right) \hat{P}\left(\boldsymbol{w}_{i}\right)}{P\left(\mathcal{D}_i\right)},
    \end{aligned}
\end{equation}
where $P_{\text{l}}\left(\mathcal{D}_i|\boldsymbol{w}_{i}\right)$ is the data likelihood, $\hat{P}\left(\boldsymbol{w}_{i}\right)$ is the local prior, and $P\left(\mathcal{D}_i\right)$ represents the evidence.
For each data sample $\left(\boldsymbol{x}_{i,n}, \boldsymbol{y}_{i,n}\right) \in \mathcal{D}_i$, $P_{\text{l}}\left(\mathcal{D}_i|\boldsymbol{w}_{i}\right)$ can be given by a parameterized Gaussian distribution:
\begin{equation}\label{eq:Locallikehood}
    \begin{split}
        P_{\text{l}}\left({\mathcal{D}_i} | \boldsymbol{w}_{i}\right)=\prod\limits_{n=1}^{N_i} P_{\text{l}}\left(\boldsymbol{y}_{i,n}|\phi\left(\boldsymbol{w}_i, \boldsymbol{x}_{i,n}\right)\right).
    \end{split}
\end{equation}
The data likelihood related to the cross-entropy loss in classification problems can be modeled by categorical distribution:
\begin{equation}\label{eq:likehood}
    \begin{split}
        P_{\text{l}}\left(\boldsymbol{y}_{i,n}|\phi\left(\boldsymbol{w}_i, \boldsymbol{x}_{i,n}\right)\right)\!=\!\exp{\left(-\mathcal{C}_E\left(\boldsymbol{y}_{i,n};\phi\left(\boldsymbol{w}_{i}, \boldsymbol{x}_{i,n}\right)\right)\right)},
    \end{split}
\end{equation}
where $\mathcal{C}_E\left(\cdot\right)$ is the cross-entropy function.
Then, we introduce the Bayesian FL process.
\subsubsection{Local training process}
From a Bayesian perspective, the primary goal is to maximize the global data likelihood $P\left(\boldsymbol{w}_{i}|\mathcal{D}_{i}\right)$ given 
$\hat{P}\left(\boldsymbol{w}_{i}\right)$.
However, directly calculating $P\left(\boldsymbol{w}_{i}|\mathcal{D}_{i}\right)$ according to (\ref{eq:BayesFunction}) involves calculating the evidence $P\left(\mathcal{D}_i\right)$, which contains an intractable multidimensional integral \cite{StructuredBayesianFederatedLearning}. 
To tackle this issue, we use a tractable variational distribution $q\left(\boldsymbol{w}_{i};\boldsymbol{\psi}_i\right)$ to estimate $P\left(\boldsymbol{w}_{i}|\mathcal{D}_{i}\right)$, where $q\left(\boldsymbol{w}_{i};\boldsymbol{\psi}_i\right) \sim \mathcal{N}\left(\boldsymbol{w}_{i};\boldsymbol{\mu}_{i,t}, \boldsymbol{\sigma}_{i,t}\right)$ with $\boldsymbol{\psi}_{i}=\left\{\boldsymbol{\mu}_{i}, \boldsymbol{\sigma}_{i}\right\}$ being the set of the mean and standard deviation of $q\left(\boldsymbol{w}_{i};\boldsymbol{\psi}_i\right)$.
The optimal variational posterior $q\left(\boldsymbol{w}_{i};\boldsymbol{\psi}_i\right)$ can be obtained by minimizing the Kullback-Leibler (KL) divergence, which is
\begin{equation}\label{eq:localupdate}
    \begin{aligned}
        \left(\boldsymbol{\mu}_{i,t}, \boldsymbol{\sigma}_{i,t}\right) = & \arg\min\limits_{\boldsymbol{\psi}_i} \mathbb{D}\left[q\left(\boldsymbol{w}_{i};\boldsymbol{\psi}_i\right)||\hat{P}\left(\boldsymbol{w}_{i}\right)P\left(\boldsymbol{w}_{i}|\mathcal{D}_{i}\right)\right],
        \\ = & \arg\min\limits_{\boldsymbol{\psi}_i} \left[\log\left(P\left(\mathcal{D}_{i}\right)\right)-\mathcal{B}\left(\boldsymbol{\psi}_i\right)\right]
        \\ \stackrel{(a)}{=} & \arg\max\limits_{\boldsymbol{\psi}_i} \mathcal{B}\left(\boldsymbol{\psi}_i\right),
    \end{aligned}
\end{equation}
where $\mathbb{D}\left(\cdot\right)$ denotes the KL divergence.
(a) holds since $\log\left(P\left(\mathcal{D}_{i}\right)\right)$ in (\ref{eq:localupdate}) is independent to $\boldsymbol{\psi}_i$ and can be regarded as a constant, and
\begin{equation}\label{eq:updateFo}
    \begin{aligned}
        \mathcal{B}\left(\boldsymbol{\psi}_i\right) = \mathbb{E}_{\boldsymbol{w}_{i}\sim\boldsymbol{\psi}_i}\log P_{\text{l}}\left(\mathcal{D}_{i}|\boldsymbol{w}_{i}\right) - \mathbb{D}\left(q\left(\boldsymbol{w}_{i};\boldsymbol{\psi}_i\right)||\hat{P}\left(\boldsymbol{w}_{i}\right)\right)
    \end{aligned}
\end{equation} 
is the evidence lower bound (ELBO) \cite{ELBO}.
The term $\log P_{\text{l}}\left(\mathcal{D}_{i}|\boldsymbol{w}_{i}\right)$ is the cross entropy.
Then, we adopt stochastic gradient descent (SGD) to update $\boldsymbol{\mu}_{i}$ and $\boldsymbol{\sigma}_{i}$, which can be expressed as
\begin{align}\label{eq:CFLGradient}
    \begin{aligned}
          \boldsymbol{\mu}_{i}' &= \boldsymbol{\mu}_{i} + \eta_t \sum\limits_{n \in \mathcal{N}_{i}^{t}} \frac {\partial \mathcal{B}\left(\boldsymbol{\psi}_i\right)}{\partial \boldsymbol{\mu}_{i}}, 
          \boldsymbol{\sigma}_{i}' = \boldsymbol{\sigma}_{i} + \eta_t \sum\limits_{n \in \mathcal{N}_{i}^{t}} \frac {\partial \mathcal{B}\left(\boldsymbol{\psi}_i\right)}{\partial \boldsymbol{\sigma}_{i}},
    \end{aligned}
\end{align}
where $\eta_t$ is the learning rate, and $\mathcal{N}_{i}^{t}$ is the subset of training data samples (i.e., minibatch) selected from the local dataset. 
\subsubsection{Model transmission process}
Then, each device $i$ exchanges the updated local models $P\left(\boldsymbol{w}_{i}|\mathcal{D}_{i}\right)$ with a subset of its neighbors for model aggregation. 
We adopt an orthogonal frequency division multiple access (OFDMA) transmission scheme.
Let $W$ be the available bandwidth for transmission and 
${\bf{p}}_{t}$ is the transmit power matrix at time slot $t$ with
$\left[{\bf{p}}_t\right]_{i,j}$ being the model transmit power of device $i$ to device $j$.
The location of device $i$ is captured by a vector $\boldsymbol{\nu}_{i,t} = [\nu_{i,t}^{1}, \nu_{i,t}^{2}]$, and $\boldsymbol{\nu}_{t}=\left[\boldsymbol{\nu}_{1,t},\cdots,\boldsymbol{\nu}_{R,t}\right]$ is the location matrix over all devices at iteration $t$.
The data rate of device $i$ transmitting the model updates to device $j$ is
\begin{equation}\label{eq:TimeDelay}
  \begin{aligned}
    l_{i,j}^{t}(\left[{\bf{U}}_{t}\right]_{i}, \left[{\bf{p}}_t\right]_{i,j})\!\! = \!\!\frac{S}{\frac{W}{|\left[{\bf{U}}_{t}\right]_{i}|}\log_{2}\left(1+\frac{\left[{\bf{p}}_t\right]_{i,j}{h_{i,j}^{t}\left(\boldsymbol{\nu}_{t}\right)}}{\frac{W}{|\left[{\bf{U}}_{t}\right]_{i}|}\sigma^2_{\emph{N}}}\right)},
  \end{aligned}
\end{equation}
where ${\bf{U}}_{t} \in \left\{0,1\right\}^{R \times R}$ is the FL transmission matrix with $\left[{\bf{U}}_{t}\right]_{i,j}=1$ implying that device $i$ will exchange its local FL model with device $j$ at iteration $t$, and $\left[{\bf{U}}_{t}\right]_{i,j}=0$, otherwise.
$\left[{\bf{U}}_{t}\right]_{i}$ is the FL transmission vector of device $i$ and 
$|\left[{\bf{U}}_{t}\right]_{i}| = \sum_{j \in \mathcal{M}} \left[{\bf{U}}_{t}\right]_{i,j}$ is the number of devices that will transmit the updated local models to device $i$. 
$h_{i,j}^{t}\left(\boldsymbol{\nu}_{t}\right)=\rho_{i,j}\left(d_{i,j}^{t}\left(\boldsymbol{\nu}_{t}\right)\right)^{-2}$ is the channel gain between device $i$ and $j$, where $\rho_{i,j}$ is the Rayleigh fading parameters.
$d_{i,j}^{t}\left(\boldsymbol{\nu}_{t}\right)$ is the distance between device $i$ and $j$.
$\sigma^2_{\emph{N}}$ represents the variance of additive white Gaussian noise. 
$S$ is the data size of the $P\left(\boldsymbol{w}_{i}|\mathcal{D}_{i}\right)$.
\subsubsection{Model aggregation process}
Given $P\left(\boldsymbol{w}_{j}|\mathcal{D}_{j}\right)$ that is received from neighbors, device $i$ updates its prior as \cite{ PeertoPeerVariationalFederatedLearningOverArbitraryGraphs}
\begin{equation}\label{eq:CFLUpdate}
    \begin{aligned}
           \hat{P}\left(\boldsymbol{w}_{i}\right) = \exp\left(\sum\limits_{j\in \mathcal{R}} \left[{\bf{A}}_t\right]_{i,j}\log P\left(\boldsymbol{w}_{j}|\mathcal{D}_{j}\right)\right),
    \end{aligned}
\end{equation}
where ${\bf{A}}_t \in \left[0,1\right]^{R \times R}$ is the model aggregation weight matrix with $\left[{\bf{A}}_t\right]_{i,j}$ denoting the aggregation weight of $P\left(\boldsymbol{w}_{j}|\mathcal{D}_{j}\right)$ and satisfying $\sum_{j\in\mathcal{R}}\left[{\bf{A}}_t\right]_{i,j}=1$.
Note that $\left[{\bf{A}}_t\right]_{i,j}>0$ if and only if $\left[{\bf{U}}_t\right]_{i,j}=1$.
Since the contribution of $P\left(\boldsymbol{w}_{j}|\mathcal{D}_{j}\right)$ to $\hat{P}\left(\boldsymbol{w}_{i}\right)$ is determined by $\left[{\bf{A}}_t\right]_{i,j}$, we can decrease the negative effects of Byzantine adversary $j$ by decreasing $\left[{\bf{A}}_t\right]_{i,j}$.
To this end, we introduce a parameter $\varrho$ to define the maximum tolerable negative impact, which is represented as the model aggregation weight for Byzantine adversaries and given by:
\begin{definition}\label{def:1}
    (Security constraint) The maximum tolerable negative impact of Byzantine adversaries is defined as $\varrho$, and then, the security constraint is given by 
    \begin{align}
        [{\bf{A}}_t]_{i,j} \leqslant \varrho, \forall i \in \mathcal{M},  \forall {t} \in \mathcal{T}, \mathrm{ if } \ j\in\mathcal{B}.
    \end{align}
\end{definition}
\begin{remark}
Note that {\bf{Definition \ref{def:1}}} specifically restricts the negative impact of Byzantine adversaries by actively limiting their contributions once they are identified (i.e., $j\in\mathcal{B}$).
\end{remark}
\subsection{Bayesian Based HBC Devices Definition}
To accurately reconstruct private data of device $i$, the HBC device $j$ needs to acquire $P^{\left(t-1\right)}\left(\boldsymbol{w}_{i}|\mathcal{D}_{i}\right)$ and $ P^{\left(t\right)}\left(\boldsymbol{w}_{i}|\mathcal{D}_{i}\right)$ which are directly exchanged among devices, and $\hat{P}^{\left(t\right)}\left(\boldsymbol{w}_{i}\right)$ which is locally accessible.
Since the HBC device $j$ cannot access $\hat{P}^{\left(t\right)}\left(\boldsymbol{w}_{i}\right)$, it can only approximate it based on the received posterior from the neighbors shared with device $i$.
In other words, to protect local data privacy, device $i$ must avoid allowing any other device $j$ to obtain $P^{\left(t-1\right)}\left(\boldsymbol{w}_{i}|\mathcal{D}_{i}\right)$, $ P^{\left(t\right)}\left(\boldsymbol{w}_{i}|\mathcal{D}_{i}\right)$, and $\hat{P}^{\left(t\right)}\left(\boldsymbol{w}_{i}\right)$ simultaneously, which can be summarized as two conditions, as detailed below:
\begin{itemize}
    \item {\bf{Condition 1}}: Device $j$ can receive both $P^{\left(t-1\right)}\left(\boldsymbol{w}_{i}|\mathcal{D}_{i}\right)$ and $P^{\left(t\right)}\left(\boldsymbol{w}_{i}|\mathcal{D}_{i}\right)$ simultaneously (i.e., $\left[{\bf{U}}_{t-1}\right]_{i,j} \odot \left[{\bf{U}}_{t}\right]_{i,j} = 1$).
    \item {\bf{Condition 2}}: The device $j$' approximated prior $P_{i,j,t}^{IP}$ of device $i$ approaches $\hat{P}^{\left(t\right)}\left(\boldsymbol{w}_{i}\right)$.
\end{itemize}
To describe {\bf{Condition 1}}, we introduce a vector $\boldsymbol{v}_{i,t}\left({\bf{U}}_{t-1}, {\bf{U}}_{t}\right) \in \left\{0, 1\right\}^{R \times 1}$ to determine the subset of devices that receives $P^{\left(t-1\right)}\left(\boldsymbol{w}_{i}|\mathcal{D}_{i}\right)$ and $ P^{\left(t\right)}\left(\boldsymbol{w}_{i}|\mathcal{D}_{i}\right)$ simultaneously, which is 
\begin{equation}\label{ModelReceivedvector}
    \boldsymbol{v}_{i,t}\left({\bf{U}}_{t-1}, {\bf{U}}_{t}\right) = \left[{\bf{U}}_{t-1}\right]_{i} \odot \left[{\bf{U}}_{t}\right]_{i},
\end{equation}
where $\odot$ is an element-wise multiplication operation,  $\left[\boldsymbol{v}_{i,t}\left({\bf{U}}_{t-1}, {\bf{U}}_{t}\right)\right]_z=1$ implies {\bf{Condition 1}} is satisfied by device $z$, and $\left[\boldsymbol{v}_{i,t}\left({\bf{U}}_{t-1}, {\bf{U}}_{t}\right)\right]_z=0$, otherwise.

Similarly, to describe {\bf{Condition 2}}, we first introduce a vector $\boldsymbol{r}_{i,j}\left({\bf{U}}_{t}\right) \in \left\{0, 1\right\}^{R \times 1}$ to determine the subset of devices that contribute to $\hat{P}^{\left(t\right)}\left(\boldsymbol{w}_{i}\right)$ and also exchange the model updates with device $j$, which is given by
\begin{equation}\label{ModelReceivedvector}
    \boldsymbol{r}_{i,j}\left({\bf{U}}_{t}\right)=\left[{\bf{U}}_{t}\right]_{i} \odot \left[{\bf{U}}_{t}\right]_{j},
\end{equation}
where the $z$-th element $\left[\boldsymbol{r}_{i,j}\left({\bf{U}}_{t}\right)\right]_z=1$ implies that device $z$' model can be used by device $j$ to approximate $P_{i,j,t}^{IP}$, and $\left[\boldsymbol{r}_{i,j}\left({\bf{U}}_{t}\right)\right]_z=0$, otherwise.
Then, device $j$ can approximate $P_{i,j,t}^{IP}$ by
\begin{equation}\label{eq:reconstructprior}
    P_{i,j,t}^{IP}
    = \exp\left(\sum\limits_{z\in \mathcal{R}}\frac{\left[\boldsymbol{r}_{i,j}\left({\bf{U}}_{t-1}\right)\right]_{z}}{|\boldsymbol{r}_{i,j}\left({\bf{U}}_{t-1}\right)|}\log P^{\left(t-1\right)}\left(\boldsymbol{w}_{z}|\mathcal{D}_{z}\right)\right).
\end{equation}
The gap between $P_{i,j,t}^{IP}$ and $\hat{P}^{\left(t+1\right)}\left(\boldsymbol{w}_{i}\right)$ is given by
\begin{equation}\label{eq:GapPtPIP}
    \begin{aligned}
        & H\left(\hat{P}^{\left(t\right)}\left(\boldsymbol{w}_{i}\right), P_{i,j,t}^{IP}\right) \\
        & \hspace{0.05cm}=
        \sum\limits_{z=1}^{V}\left(\left[\boldsymbol{\mu}_{i,t}\right]_z-\left[\boldsymbol{\mu}_{i,j,t}^{IP}\right]_z\right)^2 +\left(\left[\boldsymbol{\sigma}_{i,t}\right]_z-\left[\boldsymbol{\sigma}_{i,j,t}^{IP}\right]_z\right)^2,
    \end{aligned}
\end{equation}
where $\boldsymbol{\mu}_{i,j,t}^{IP}$ and $\boldsymbol{\sigma}_{i,j,t}^{IP}$ are respectively the mean vector and standard deviation vector of $P_{i,j,t}^{IP}$
\footnote{According to \cite{UnderstandingDeepGradientLeakage}, the reconstruction data error, represented by the L2 distance between the reconstructed data based on the Deep Leakage from Gradients (DLG) \cite{DLG, BatchDLG, AIJack} method and the origin data, increases as $H\left(\hat{P}^{\left(t\right)}\left(\boldsymbol{w}_{i}\right), P_{i,j,t}^{IP}\right)$ increases.}.
$V$ is the dimension of model updates.
 
Then, we introduce a parameter $\psi$ to define the maximum tolerable $H\left(\hat{P}^{\left(t\right)}\left(\boldsymbol{w}_{i}\right), P_{i,j,t}^{IP}\right)$ to protect data privacy when {\bf{Condition 1}} is satisfied, which can be summarized by the following definition. 
\begin{definition}\label{def:2}
(Privacy constraint) The maximum tolerable privacy requirement is defined as $\psi$.
Consequently, by combining the {\bf{Condition 1}} and {\bf{Condition 2}}, the privacy constraint is given by 
\begin{equation}
    H\left(\hat{P}^{\left(t\right)}\left(\boldsymbol{w}_{i}\right), P_{i,j,t}^{IP}\right) \geq \psi, \mathrm{ if } \ \left[\boldsymbol{v}_{i,t}\right]_{j}=1, 
    \forall i,j\in\mathcal{M},
\end{equation}
where $\boldsymbol{v}_{i,t}$ is short for $\boldsymbol{v}_{i,t}\left({\bf{U}}_{t-1}, {\bf{U}}_{t}\right)$.
\end{definition}
\begin{remark}
Note that {\bf{Definition \ref{def:2}}} establishes the permissible level of approximation of critical information of the local prior without compromising privacy. The specific value of $\psi$ should be determined based on the requirements of each device.
\end{remark}
\subsection{Problem Formulation}
We formulate our optimization problem whose goal is to minimize the DFL training loss while jointly considering transmission delay, transmission power, security, and privacy constraints. 
The optimization problem is formulated as
\begin{equation}\label{eq:max1}
    \begin{aligned}
    \mathop {\min }\limits_{{\bf{U}}_t, {\bf{A}}_t, {{\bf{p}}}_t} & \frac{1}{M} \sum\limits_{i=1}^{M} F\left(\boldsymbol{w}_{i}\right)   
    \end{aligned}
\end{equation}
\begin{align}\label{c1}
    \setlength{\abovedisplayskip}{-20 pt}
    \setlength{\belowdisplayskip}{-20 pt}
    {\rm{s.t.}} \,\, 
    & [{\bf{A}}_t]_{i,j} \leqslant \varrho, \forall i \in \mathcal{M},\forall j \in \mathcal{R}, \forall {t} \in \mathcal{T}, \text{ if } j\in\mathcal{B},  \tag{\theequation a} \\
    & H\left(\hat{P}^{\left(t\right)}\left(\boldsymbol{w}_{i}\right), P_{i,j,t}^{IP}\right) \geq \psi, \forall i \in\mathcal{M}, \forall j\in\mathcal{R},\notag \\
    & \qquad\qquad\qquad\qquad\qquad \text{ if } \left[\boldsymbol{v}_{i,t}\right]_{j}=1, \forall {t} \in \mathcal{T}, \tag{\theequation b} \\
    & l_{i,j}^{t}\left(\left[{\bf{U}}_{t}\right]_{i}, \left[{\bf{p}}_t\right]_{i,j}\right) \leqslant \Gamma,\forall i,j \in \mathcal{R},\forall {t} \in \mathcal{T}, \tag{\theequation c} \\
    & \sum\limits_{j=1}^R \left[{\bf{p}}_{t}\right]_{i,j} \leqslant p_{\text{max}}, \forall i \in \mathcal{R}, \forall {t} \in \mathcal{T},\tag{\theequation d}    
\end{align}
where 
(\ref{eq:max1}a) is the security constraint, as shown in {\bf{Definition \ref{def:1}}}.
(\ref{eq:max1}b) is the privacy constraint, as shown in {\bf{Definition~\ref{def:2}}}.
(\ref{eq:max1}c) is the DFL model transmission delay constraint where $\Gamma$ is the maximum transmission delay per iteration. 
(\ref{eq:max1}d) is the transmit power constraint where $p_{\text{max}}$ is the maximum transmit power.

The problem in (\ref{eq:max1}) is challenging to solve by traditional algorithms due to the following reasons.
First, each device must collect all devices' local models so as to minimize the DFL training loss while satisfying the security, privacy, and transmission delay constraints.
However, each device can only collect neighbors' information in the DFL framework.
Thus, we must collect global information and then find the optimal ${\bf{U}}_t, {\bf{A}}_t$, and ${{\bf{p}}}_t$ when using traditional iterative methods which introduce addition delay and computational overhead.
Also, as the device's location and channel information vary, each device must re-execute the iterative methods, which leads to additional overhead.
Second, constraint (\ref{eq:max1}a) needs to be satisfied when $j\in \mathcal{B}$ is achieved first.
Moreover, before satisfying constraint (\ref{eq:max1}a), the devices need first to determine the $j\in \mathcal{B}$.
Thus, each device $i$ must identify the Byzantine adversaries among its neighbors based on partial information (e.g., device connection and model updates), which cannot be solved by traditional detection methods that collect and compare all local models in the network.

\section{Security, Privacy, and Convergence Acceleration Analysis}



\subsection{Security Constraint Analysis}
To satisfy the security constraint in (\ref{eq:max1}a), device $i$ must first identify the Byzantine adversary $j$ among its neighboring devices (i.e., $j\in\mathcal{B}$).
Note that, the key difference between Byzantine adversaries and standard training devices is that the Byzantine adversaries send random model updates without local training while the standard training devices send the model updates that are optimized according to (\ref{eq:localupdate}) and (\ref{eq:updateFo}).
Hence, to distinguish Byzantine adversaries, device $i$ needs to verify whether $P_{\text{l}}^{\left(t\right)}\left({\mathcal{D}_j} | \boldsymbol{w}_{j}\right)$ and $\mathbb{D}\left(P^{\left(t\right)}(\boldsymbol{w}_{j} | \mathcal{D}_j) || \hat{P}^{\left(t\right)}\left(\boldsymbol{w}_{j}\right)\right)$ in  (\ref{eq:updateFo}) are optimized.
However, since $P_{\text{l}}^{\left(t\right)}\left({\mathcal{D}_j} | \boldsymbol{w}_{j}\right)$ is locally accessible, the only feasible approach is to verify whether $\mathbb{D}\left(P^{\left(t\right)}(\boldsymbol{w}_{j} | \mathcal{D}_{j}) || P_{j,i,t}^{IP}\right)$ is optimized.
To this end, we define a threshold $\kappa\left({\bf{U}}_{t}\right)$ to identify Byzantine adversaries, which is 
\begin{definition}\label{def:3}
(Identify Byzantine adversaries) The threshold of device $i$ to identify Byzantine adversary $j$ is defined as $\kappa\left({\bf{U}}_{t}\right)$, then, we introduce Byzantine identification matrix ${\bf{B}}_{t}\in\left\{0,1\right\}^{R \times R}$ which is given by 
\begin{equation}\label{class:Byzantine}
    \begin{aligned}
        & \left[{\bf{B}}_{t}\right]_{i,j}
        = \left\{
        \begin{aligned}
        & 1, \mathrm{if} \ \mathbb{D}\left(P^{\left(t\right)}\left(\boldsymbol{w}_{j}| {\mathcal{D}}_{j}\right)||P_{j,i,t}^{IP}\right) \geq \kappa({\bf{U}}_{t}), \\
        & 0, \mathrm{ otherwise, } \\
        \end{aligned}
        \right. 
    \end{aligned}
\end{equation}
where $\left[{\bf{B}}_{t}\right]_{i,j}=1$ implies that device $i$ classifies device $j$ as a Byzantine adversary, $\left[{\bf{B}}_{t}\right]_{i,j}=0$, otherwise.
\end{definition}
Note that, $\kappa\left({\bf{U}}_{t}\right)$ is affected by $\left[{\bf{U}}_{t}\right]_{j}$ and needs to be determined during the training process.
Then (\ref{eq:max1}a) can be rewritten as
\begin{align}\label{eq:rewritten18a}
    \mathbbm{1}_{\left\{\left[{\bf{B}}_{t}\right]_{i,j}\right\}}[{\bf{A}}_t]_{i,j} \leqslant \varrho, \forall i \in \mathcal{M}, \forall j \in \mathcal{R}, \forall {t} \in \mathcal{T},
\end{align}
where $\mathbbm{1}_{\left\{x\right\}}=x$ if $x>0$, $\mathbbm{1}_{\left\{x\right\}}=0$, otherwise.
\subsection{DFL Convergence Analysis}
Next, we will analyze how ${\bf{A}}_t$ affects the performance of DFL in (\ref{eq:max1}).
To find the relationship between ${\bf{A}}_t$ and the DFL performance, we must first analyze the convergence rate of DFL.
To this end, we first make the following assumptions, as done in \cite{SocialLearningandDistributedHypothesisTesting, PeertoPeerVariationalFederatedLearningOverArbitraryGraphs}
\begin{itemize}
\item {\it{Assumption$~1$}:}
    For each device $i \in \mathcal{M}$, there exists a  $\boldsymbol{\psi}_i$ such that $q\left(\boldsymbol{w}_{i};\boldsymbol{\psi}_i\right)=\hat{P}\left(\boldsymbol{w}_{i}\right)P\left(\boldsymbol{w}_{i}|\mathcal{D}_{i}\right)$. 
    In this case, as shown in (\ref{eq:localupdate}), it is possible to derive the expected loss to zero.
\item {\it{Assumption$~2$}:}
    For each device $i \in \mathcal{M}$, the gap between $P_{\text{l}}\left(\mathcal{D}_{i}|\boldsymbol{w}_{i}\right)$ and $P\left(\mathcal{D}_{i}\right)$ is bounded, we have
    \begin{equation}
        \begin{aligned}
        \left\|\log\frac{P_{\text{l}}\left(\mathcal{D}_{i}|\boldsymbol{w}_{i}\right)}{P\left(\mathcal{D}_{i}\right)}\right\| \leq L, \forall i\in \mathcal{M}.
        \end{aligned}
    \end{equation}
\end{itemize}
These assumptions are natural where Assumption~$1$ assumes the learning problem is realizable for all devices, and Assumption~$2$ assumes an upper bound of $\left\|\log\frac{P_{\text{l}}\left(\mathcal{D}_{i}|\boldsymbol{w}_{i}\right)}{P\left(\mathcal{D}_{i}\right)}\right\|$ exists.

To analyze the convergence rate of DFL, we can analyze the convergence of posterior $P^{\left(t\right)}\left(\boldsymbol{w}_{i}|\mathcal{D}_{i}\right)$.
However, since the devices' posterior depends on a subset of its neighbors, the posterior of each device is inconsistent. 
Then, we introduce $P^{\star\left(t\right)}\left(\boldsymbol{w}\right)$ as the global prior which is the average aggregation of all standard training devices' posterior  $P^{\left(t\right)}\left(\boldsymbol{w}_{i}|\mathcal{D}_{i}\right)$, which satisfies:
\begin{equation}
    \begin{aligned}
    \log P^{\star\left(t\right)}\left(\boldsymbol{w}\right) =  \sum\limits_{j\in \mathcal{M}} \frac{1}{M}\log P\left(\boldsymbol{w}_{j}|\mathcal{D}_{j}\right).
    \end{aligned}
\end{equation}
Moreover, due to the Assumption~$1$, the prior  $P^{\star\left(t\right)}\left(\boldsymbol{w}\right)$ is the optimal prior at iteration $t$.
To this end, we can analyze the convergence rate of DFL by analyzing how ${\bf{A}}_t$ affects the upper bound of the gap between $ P^{\left(t\right)}\left(\boldsymbol{w}_{i}|\mathcal{D}_{i}\right)$ and the optimal prior $P^{\star\left(t\right)}\left(\boldsymbol{w}\right)$, which is shown in {\bf{Theorem~\ref{thm:theorem1}}}.
\begin{theorem}\label{thm:theorem1}
Given the model aggregation matrix ${\bf{A}}_t$ at each iteration, the upper bound of the gap between each local posterior $P^{\left(t+1\right)}\left(\boldsymbol{w}_{i}|\mathcal{D}_{i}\right)$ and the average posterior $P^{\star\left(t+1\right)}\left(\boldsymbol{w}\right)$ is given by
\begin{equation}
    \begin{aligned}
        \frac{1}{M}\sum\limits_{i=1}^{M}||\log P^{\left(t+1\right)}\left(\boldsymbol{w}_{i}|\mathcal{D}_{i}\right) \!-\!&\log P^{\star\left(t+1\right)}\left(\boldsymbol{w}\right)||^{2} \\
        & \leq 
        \left(L\sqrt{M} \sum\limits_{\tau=1}^{t} \sqrt{\xi_{\tau}}\right)^{2},
    \end{aligned}
\end{equation}
where $\xi_{\tau}\in\left[0,1\right]$ is the second largest eigenvalue of $\prod\limits_{z=\tau}^{t}{\bf{A}}_{z}^{\top}\left(\prod\limits_{z=\tau}^{t}{\bf{A}}_{z}^{\top}\right)^{\top}$.
\end{theorem}
\begin{proof}
    See Appendix \ref{proof:Theorem1}.
\end{proof}

From Theorem~\ref{thm:theorem1}, we can see that the gap between $P^{\left(t+1\right)}\left(\boldsymbol{w}_{i}|\mathcal{D}_{i}\right)$ and $P^{\star\left(t+1\right)}\left(\boldsymbol{w}\right)$ decreases as the second largest eigenvalues of $\prod\limits_{z=\tau}^{t}{\bf{A}}_{z}^{\top}\left(\prod\limits_{z=\tau}^{t}{\bf{A}}_{z}^{\top}\right)^{\top}$ decreases.
Thus, we can minimize $\xi_{\tau}$ to decrease the gap between the DFL posterior $P^{\left(t+1\right)}\left(\boldsymbol{w}_{i}|\mathcal{D}_{i}\right)$ and the optimal DFL prior $P^{\star\left(t+1\right)}\left(\boldsymbol{w}\right)$.
Then, problem (\ref{eq:max1}) can be rewritten as 
\begin{equation}\label{eq:max2}
    \begin{aligned}
    \mathop {\min }\limits_{{\bf{U}}_t, {\bf{A}}_t, {{\bf{p}}}_t} & \sum\limits_{\tau=1}^{t} \xi_{\tau}  
    \end{aligned}
\end{equation}
\begin{align}\label{c2}
    \setlength{\abovedisplayskip}{-20 pt}
    \setlength{\belowdisplayskip}{-20 pt}
    {\rm{s.t.}} \,\, 
    & \mathbbm{1}_{\left\{\left[{\bf{B}}_{t}\right]_{i,j}\right\}}[{\bf{A}}_t]_{i,j} \leqslant \varrho, \forall i \in \mathcal{M}, \forall j \in \mathcal{R}, \forall {t} \in \mathcal{T}, \tag{\theequation a} \\
    & (\ref{eq:max1}\text{b})-(\ref{eq:max1}\text{d}) \notag
\end{align}
where (\ref{eq:max2}a) is due to the (\ref{eq:rewritten18a}).

To solve (\ref{eq:max2}), one promising solution is to use an RL based algorithm that can adapt to the dynamic environment.
However, the traditional RL algorithms that rely on fixed input dimensions may be invalid for addressing the problem with varying input dimensions caused by the varying number of neighbors in DFL.
Moreover, since the security, privacy, and convergence acceleration as shown in (\ref{eq:max2}) are affected by ${\bf{A}}_{t}$ at previous iterations, devices need to jointly optimize ${\bf{A}}_t$ and ${\bf{U}}_t$ by considering the correlation between consecutive iterations.
To tackle these issues, one promising solution is to replace traditional fixed-input neural networks in RL with the neural networks that combine GNNs \cite{ToNGNN} with recurrent neural networks (RNN) \cite{ToNRNN} so as to adapt dynamic input dimensions and jointly handling the spatial and temporal information.
Hence, we propose a distributed RL algorithm that combines GNN and RNN to solve the problem (\ref{eq:max2}).

\section{Proposed GNN Based RL Method}
In this section, our goal is to propose a novel GNN based RL algorithm that enables each device to determine the ${\bf{U}}_t$, ${\bf{A}}_t$, and ${\bf{p}}_t$ independently based on the partial information (e.g., received posterior, approximated prior, and device connection) from neighbors to achieve security, privacy, and convergence acceleration which satisfy the transmission constraint.
Since ${\bf{A}}_t$ and ${\bf{p}}_t$ can only be optimized after ${\bf{U}}_t$ is determined, a two stages distributed approach is proposed.
In the first stage, we propose a GNN based RL algorithm which enables each device $i$ to optimize $\left[{\bf{U}}_t\right]_{i}$ by extracting the spatial and temporal features based on partial information.
In the second stage, ${\bf{p}}_t$ is determined by convex optimization, and $\left[{\bf{A}}_t\right]_{i}$ is determined based on the received model updates.
Compared to standard RL algorithms \cite{PerformanceOptimizationforSemanticCommunicationsAnAttentionBasedReinforcementLearningApproach}, our proposed GNN based RL approach applies GNN and RNN serves as the backbone to capture dependencies in the spatial dimension including wireless channel and neighbor' features and applies RNN to capture correlations in the time series dimension, then the RL method is used to infer the optimal action (i.e., ${\bf{U}}_t$ and ${\bf{A}}_t$).

\begin{figure*}[t]
  \centering
  \includegraphics[width=15.5cm]{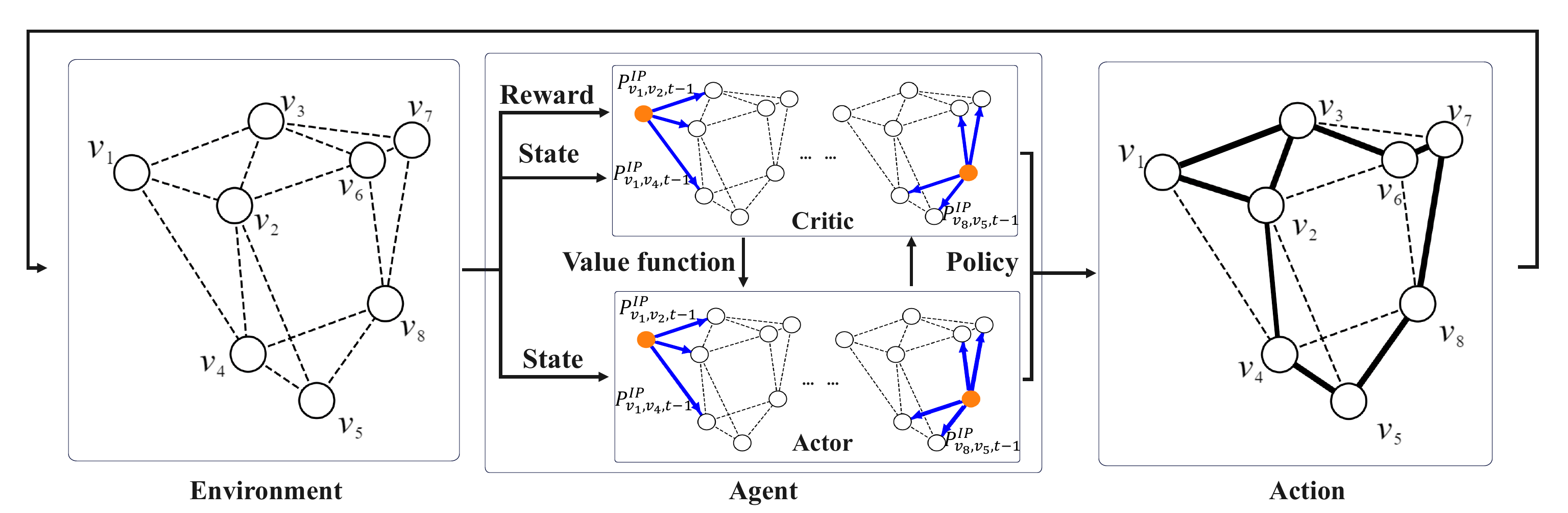}
  \centering
  \caption{Illustration of our proposed GNN-based algorithm.}\label{fig2}
\end{figure*}

Next, we will first introduce the use of GNN-RNN based models to estimate ${\bf{U}}_t$.
Then, devices exchange the model updates and allocate transmission power ${\bf{p}}_t$  based on ${\bf{U}}_t$ and determine ${\bf{A}}_t$ based on the received model updates.
Finally, we will show the RL algorithm to update the proposed GNN-RNN based models as shown in Fig. \ref{fig2}.
\subsection{GNN-RNN Based Model Design}
\begin{figure}[t]
  \centering
  \includegraphics[width=8cm]{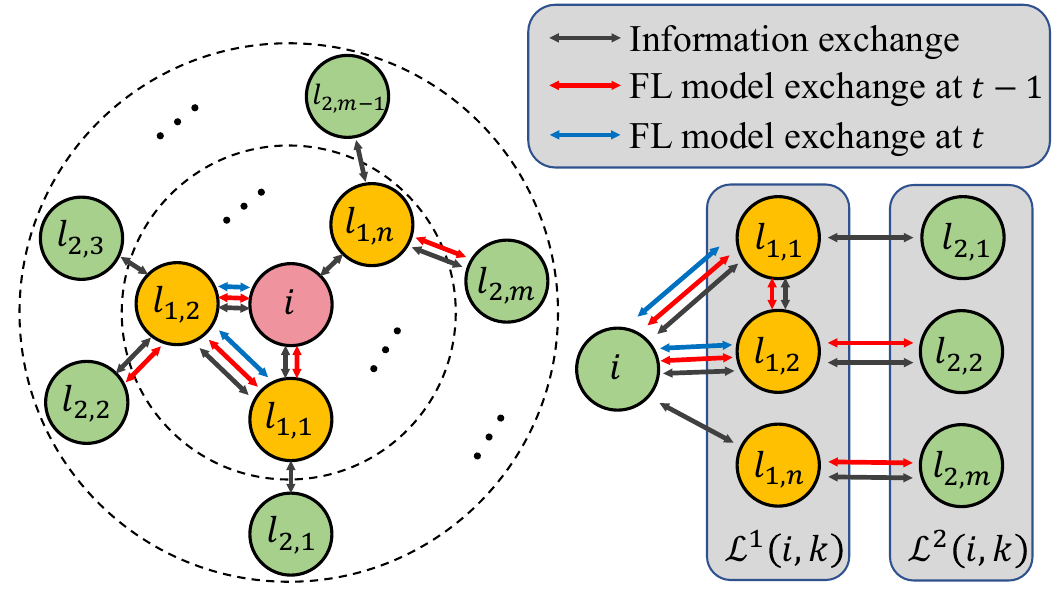}
  \centering
  \vspace{-0.1cm}
  \caption{Exchange information and FL model with neighbors.}\label{neighbors}
  \vspace{-0cm}
\end{figure}
\subsubsection{Input} To define the input of the GNN model at each device $i$, we first denote the set of $k$ nearest neighbors of device $i$ in iteration $t$ by $\mathcal{L}^1\left(i,k,t\right)$, i.e. first-hop devices. 
Let $\mathcal{L}^2\left(i,k,t\right)$ be the set of $k$ nearest neighbors of the devices in $\mathcal{L}^1\left(i,k,t\right)$, i.e. second-hop devices. 
For example, in Fig. \ref{neighbors}, $\mathcal{L}^1\left(i,k,t\right)=\{l_{1,1},\cdots,l_{1,n}\}$ while $\mathcal{L}^2\left(i,k,t\right)=\{l_{2,1},\cdots,l_{2,m}\}$.
Let $\mathcal{L}\left(i,k,t\right)=\mathcal{L}^1\left(i,k,t\right) \cup \mathcal{L}^2\left(i,k,t\right) \cup \{i\}$ and $|\mathcal{L}\left(i,k,t\right)|$ are the number of devices in $\mathcal{L}\left(i,k,t\right)$. 
Then, the features of the connection between device $i$ and device $j$ are defined as $\boldsymbol{E}_{i,j,t}=\left[H\left(\hat{P}^{\left(t\right)}\left(\boldsymbol{w}_{i}\right), P_{i,j,t}^{IP}\right), \left[{\bf{A}}_{t-1}\right]_{i,j}\right]$ being the vector containing $H\left(\hat{P}^{\left(t\right)}\left(\boldsymbol{w}_{i}\right), P_{i,j,t}^{IP}\right)$ and $\left[{\bf{A}}_{t-1}\right]_{i,j}$.
Then, the edge features $\boldsymbol{E}_{i,j,t}$ need to be normalized which is given by 
\begin{equation}\label{eq:GAT1}
    \begin{aligned}
        {\tilde{\boldsymbol{E}}}_{i,j,t} = \frac{\boldsymbol{E}_{i,j,t}}{\sum_{k=0}^{|\mathcal{L}^1\left(i,k,t\right)|}\boldsymbol{E}_{i,k,t}}. \\
    \end{aligned}
\end{equation}
The inputs to the GNN of each device $i$ at iteration $t$ are denoted by ${\tilde{\boldsymbol{E}}}_{i,j,t}$,  $\boldsymbol{\nu}_t$, and neighbor relationship $\mathcal{L}^1\left(i,k,t\right)$.
\subsubsection{Spatial features extraction} 
The GNN hidden layers to extract spatial features consist of a) graph attention (GAT) layer I, b) fully connected (FC) layer II, c) GAT layer III, and d) FC layer IV, which are expressed as
\begin{equation}\label{eq:GAT1}
    \begin{aligned}
        \boldsymbol{q}_{i,t} = \psi \left(\frac{1}{K_1}\sum \limits_{k_1 = 1}^{K_1} \sum \limits_{k \in \mathcal{L}^1\left(i,j\right)} \alpha_{i,k}^{k_1} \boldsymbol{W}^{k_1} \left[\boldsymbol{\nu}_{k}^{t}, {\tilde{\boldsymbol{E}}}_{i,k,t}\right]\right), \\
    \end{aligned}
\end{equation}
\begin{equation}\label{eq:Hidden_full_connection1}
    \begin{aligned}
        \boldsymbol{z}_{i,t} = \psi \left( \boldsymbol{W}^{K_2} \boldsymbol{q}_{i,t}\right), \\
    \end{aligned}
\end{equation}
\begin{equation}\label{eq:GAT2}
    \begin{aligned}
        \boldsymbol{q}'_{i,t} = \psi \left(\frac{1}{K_3}\sum \limits_{k_3 = 1}^{K_3} \sum \limits_{k \in \mathcal{L}^1\left(i,j\right)} \alpha_{i,k}^{k_3} \boldsymbol{W}^{k_3} \boldsymbol{z}_{k,t} \right), \\
    \end{aligned}
\end{equation}
\begin{equation}\label{eq:Hidden_full_connection2}
    \begin{aligned}
        \boldsymbol{z}'_{i,t} = \psi \left( \boldsymbol{W}^{K_4} \boldsymbol{q}'_{i,t}\right), \\
    \end{aligned}
\end{equation}
where $\boldsymbol{q}_{i,t}$ and $\boldsymbol{q}'_{i,t}$ are respectively the output of the GAT layer I and III, $\boldsymbol{z}_{i,t}$ and $\boldsymbol{z}'_{i,t}$ are respectively the output of the FC layers II and FC layers IV. 
$K_1$ and $K_3$ are respectively the numbers of attention heads in the first GAT layer and the second GAT layer. 
$\boldsymbol{W}^{k_1} \in \mathbb{R}^{4\times V_1}$ and $\boldsymbol{W}^{k_3} \in \mathbb{R}^{V_1\times V_2}$ are the GAT model parameters used to transform the input features into higher-level features. 
$\boldsymbol{W}^{k_2} \in \mathbb{R}^{V_1\times V_1}$ and $\boldsymbol{W}^{k_4} \in \mathbb{R}^{V_2\times V_2}$ are the two FC layers. 
Here, we need to note that $\boldsymbol{z}_{i,t}$ updated by device $i$ would be shared with devices in $\mathcal{L}\left(i,k\right)$ for computation in (\ref{eq:Hidden_full_connection1}).
$\alpha_{i,j}^{k_1}$ and $\alpha_{i,j}^{k_3}$ are the self-attention coefficients of the first and second GAT layers, and $\psi \left(\cdot\right)$ is the leaky rectified linear unit function. 
To define the attention coefficients $\alpha_{i,j}^{k_1}$, we need to define the FC function \cite{GAT} 
that can compute the correlation between two neighbor devices as follows:
\begin{equation}
    s_{i,j} = \psi \left(\boldsymbol{W}^{{\alpha}_1}\times\left(\boldsymbol{W}^{k_1} \left[\boldsymbol{\nu}_{i}^{t}, {\tilde{\boldsymbol{E}}}_{i,j,t}\right], \boldsymbol{W}^{k_1} \left[\boldsymbol{\nu}_{j}^{t}, {\tilde{\boldsymbol{E}}}_{j,i,t}\right]\right)\right).
\end{equation}
Subsequently, we normalize the correlation across each device in $\mathcal{L}^1\left(i,k,t\right)$ as the attention coefficient using the softmax function to calculate the self-attention coefficient of device $i$ and device $j$ \cite{AttentionisAllYouNeed}:
\begin{equation}
    \alpha_{i,j}^{k_1} = \frac{e^{s_{i,j}}}{\sum_{l=1}^{|\mathcal{L}^1\left(i,k,t\right)|}e^{s_{i,l}}}, \forall j \in \mathcal{L}^1\left(i,k,t\right).
\end{equation}
\subsubsection{Temporal features extraction}
The hidden RNN layers to extract temporal features between $\boldsymbol{z}'_{i,t}$ in the time series iterations consist of a) Update gate layer, b) Reset gate layer, c) Candidate representation layer, and d) Final update layer, which are expressed as
\begin{equation}\label{eq:NewInfo}
    \begin{aligned}
        \boldsymbol{e}_{i,t} = \psi \left(\boldsymbol{W}^{e_1} \boldsymbol{z}'_{i,t}+\boldsymbol{W}^{e_2} \boldsymbol{h}_{i,t-1}\right), \\
    \end{aligned}
\end{equation}
\begin{equation}\label{eq:OldInfo}
    \begin{aligned}
        \boldsymbol{r}_{i,t} = \psi \left(\boldsymbol{W}^{r_1} \boldsymbol{z}'_{i,t}+\boldsymbol{W}^{r_2} \boldsymbol{h}_{i,t-1}\right), \\
    \end{aligned}
\end{equation}
\begin{equation}\label{eq:AggInfo}
    \begin{aligned}
        \tilde{\boldsymbol{h}}_{i,t} = \delta\left(\boldsymbol{W}^{a} \boldsymbol{z}'_{i,t}+U\left(\boldsymbol{r}_{i,t}\odot \boldsymbol{h}_{i,t-1}\right)\right), \\
    \end{aligned}
\end{equation}
\begin{equation}\label{eq:Choose}
    \begin{aligned}
        \boldsymbol{h}_{i,t} = \left(1-\boldsymbol{e}_{i,t}\right)\odot\boldsymbol{h}_{i,t-1}+\boldsymbol{e}_{i,t}\odot\tilde{\boldsymbol{h}}_{i,t}, \\
    \end{aligned}
\end{equation}
where $\boldsymbol{e}_{i,t}$ is the update gate and $\boldsymbol{r}_{i,t}$
being the forget gate.
$\boldsymbol{h}_{i,t-1}$ and $\boldsymbol{h}_{i,t}$ are the extracted temporal features of iteration $t-1$ and $t$, respectively.
$\tilde{\boldsymbol{h}}_{i,t}$ is the candidate representation which aggregated the information of iteration $t-1$ and $t$.
$\delta\left(\cdot\right)$ is the hyperbolic tangent function.
$\boldsymbol{W}^{e_1}, \boldsymbol{W}^{e_2}, \boldsymbol{W}^{r_1}$, $\boldsymbol{W}^{r_2}$, and $\boldsymbol{W}^{a}$ are trainable parameters.

\subsubsection{Output}
Given the extracted features $\boldsymbol{z}'_{i,t}$ and $\boldsymbol{h}_{i,t}$, devices need to determine the ${\bf{U}}_t$ first, and then exchange the local model updates based on the device connection scheme ${\bf{U}}_t$.
Then, devices determine ${\bf{A}}_t$ based on $\mathbb{D}\left(P^{\left(t\right)}\left(\boldsymbol{w}_{j}| {\mathcal{D}}_{j}\right)||P_{j,i,t}^{IP}\right)$.

To determine ${\bf{U}}_t$, we calculate the device connection probability $\mu_{i,j}^{t}$ by applying the softmax function based on $\boldsymbol{z}'_{i,t}$ which is given by
\begin{equation}
    \mu_{i,j}^{t} = \frac{e^{\boldsymbol{z}'_{i,t}\times\boldsymbol{z}^{'\top}_{j,t}}}{\sum_{l=1}^{|\mathcal{L}^1\left(i,k,t\right)|} e^{\boldsymbol{z}'_{i,t}\times\boldsymbol{z}^{'\top}_{l,t}}}, \forall j \in \mathcal{L}^1\left(i,k,t\right).
\end{equation} 
Given the estimated connection probabilities $\boldsymbol{\mu}_{i}^t=\left[\mu_{i,1}^{t},\cdots,\mu_{i,|\mathcal{L}^1\left(i,k,t\right)|}^{t}\right]$, we employ an applying scheme that enables each device to determine device connections independently which is summarized as follows: 
\begin{enumerate}
  \item Each device $i$ sends a connection request to its neighbor $j$ that has the highest connection probability $\mu_{i,j}^{t}\in\boldsymbol{\mu}_{i}^{t}$ when its remaining transmission power is sufficient to transmit its models, and receives the request from its neighbors.  
  \item Each device builds a connection with the neighbors that has the highest connection probability request. 
\end{enumerate}
Steps 1)-2) are repeated until the remaining transmit power of each device $i$ is not enough to satisfy the communication requirements.
The optimal transmit power allocation of $p_{i,j,t}$ for device $i$ for transmitting its model updates to device $j$ is given by 
\begin{equation}
    \left[{\bf{p}}_t\right]_{i,j} = \frac{W\sigma^2_{\emph{N}}}{|\left[{\bf{U}}_{t}\right]_{i}|{h_{i,j}^{t}\left(\boldsymbol{\nu}_{t}\right)}}\left(2^{\frac{S{W}}{\Gamma|\left[{\bf{U}}_{t}\right]_{i}|}}-1\right).
\end{equation} 

Then, devices exchange the model updates based on $\left[{\bf{U}}_{t}\right]_{i}$, and subsequently, $\mathbb{D}\left(P^{\left(t\right)}\left(\boldsymbol{w}_{j}| {\mathcal{D}}_{j}\right)||P_{j,i,t}^{IP}\right)$ can be calculated.
Given model $\mathbb{D}\left(P^{\left(t\right)}\left(\boldsymbol{w}_{j}| {\mathcal{D}}_{j}\right)||P_{j,i,t}^{IP}\right)$, a GNN layer is employed to determine the model aggregation weight, which is given by
\begin{equation}
    \begin{aligned}
        h_{i,j,t} = \psi \left( \boldsymbol{W}^{K_5} \times \left[\boldsymbol{h}_{i,t}, \mathbb{D}\left(P^{\left(t\right)}\left(\boldsymbol{w}_{j}| {\mathcal{D}}_{j}\right)||P_{j,i,t}^{IP}\right)\right]\right), \\
    \end{aligned}
\end{equation}
where $\boldsymbol{W}^{K_5}$ is a trainable parameter.
Then, ${\bf{A}}_{t}$ is determined by a softmax function based on $h_{i,j,t}$, which is given by
\begin{equation}
    \left[{\bf{A}}_{t}\right]_{i,j} = \frac{e^{\mathbbm{1}_{\left\{\left[{\bf{U}}_{t}\right]_{i,j}\right\}}h_{i,j,t}}}{\sum_{l=1}^{|\left[{\bf{U}}_t\right]_{i}|} e^{\mathbbm{1}_{\left\{\left[{\bf{U}}_{t}\right]_{i,l}\right\}}h_{i,l,t}}}, \forall j \in \mathcal{L}^1\left(i,k,t\right).
\end{equation}

\subsection{PPO Based Algorithm Design}
Then, we propose a PPO-based RL method for GNN-RNN based model update.
The proposed PPO method consists of six components: a) agent, b) action, c) state, d) policy, e) critic, and f) reward, which are specified as follows:
\begin{itemize}
    \item {\bfseries Agent:} Our agent is the mobile devices in $\mathcal{R}$ that determines ${\bf{U}}_{t}$ and ${\bf{A}}_{t}$ independently.
    \item {\bfseries Action:} We define the action of agent $i$ at iteration $t$ as the $\boldsymbol{\mu}_{i,j}^{t}\in\left[0,1\right]$ and $\left[{\bf{A}}_{t}\right]_{i,j}\in\left[0,1\right]$ that represents the device connection probability and the model aggregation weight. 
    \item {\bfseries State:} The state observed by device $i$ at iteration $t$, defined as set $\mathcal{S}_{i,t}$ that consists of: 
    1) the neighboring devices' location information $\left[\boldsymbol{\nu}_{1,t},\cdots,\boldsymbol{\nu}_{|\mathcal{L}\left(i,k,t\right)|,t}\right]$, 
    2) the historical model aggregation weight matrix  $\left[{\bf{A}}_{t}\right]_{i}$,
    3) the KL divergence between the approximated prior and the local prior $H\left(\hat{P}^{\left(t\right)}\left(\boldsymbol{w}_{i}\right), P_{i,j,t}^{IP}\right)$,
    4) the distance between the approximated prior and the received posterior $\mathbb{D}\left(P^{\left(t\right)}\left(\boldsymbol{w}_{j}| {\mathcal{D}}_{j}\right)||P_{j,i,t}^{IP}\right)$ if $\left[{\bf{U}}_{t}\right]_{i,j} = 1$.
    \item {\bfseries Policy:} The policy is the probability of the agent $i$ for choosing $\mu_{i,j}^{t}\in\left[0,1\right]$ and $\left[{\bf{A}}_{t}\right]_{i,j}\in\left[0,1\right]$ given the state $\mathcal{S}_{i,t}$. 
    The PPO algorithm uses a GNN-RNN based model parameterized by $\boldsymbol{\theta}_t$ to build the relationship between the input state $\mathcal{S}_{i,t}$ and the output policy that can achieve the minimum $\xi_{\tau}$ and satisfy the constraints in (\ref{eq:max2}), which is also called the actor. 
    Then, the policy can be expressed as $\boldsymbol{\pi}_{\boldsymbol{\theta}_t}\left(\mathcal{S}_{i,t}, \boldsymbol{\mu}_{i}^{t}, \left[{{\bf{A}}_{t}}\right]_{i}\right) = P\left(\boldsymbol{\mu}_{i}^{t}, \left[{{\bf{A}}_{t}}\right]_{i}|\mathcal{S}_{i,t}\right)$.
    \item {\bfseries Critic:} The critic $\boldsymbol{V}_{\boldsymbol{\theta}_{t}^{'}}\left(\mathcal{S}_{i,t}\right)$ in the proposed method is a function to estimate the value-function of the current policy for a given state $\mathcal{S}_{i,t}$, which is the GNN-RNN based model which parameterized by $\boldsymbol{\theta}_{t}^{'}$.
    \item {\bfseries Reward:}  The reward of choosing action $\boldsymbol{\mu}_{i}^{t}, {\bf{A}}_{t}$ based on state $\mathcal{S}_{i,t}$ is given by
    \begin{equation}\label{eq:reward}
    \begin{aligned}
        r\left(\boldsymbol{\mu}_{i}^{t}, {{\bf{A}}_{t}}|\mathcal{S}_{i,t}\right) = & 
        - \xi_{t} - \sum\limits_{j=1}^{R} \mathbbm{1}_{\left\{ \mathbbm{1}_{\{[{\bf{B}}_{t}]_{i,j}\}} [{\bf{A}}_t]_{i,j} - \varrho\right\}} \\ 
        & - \sum\limits_{j=1}^{R} \mathbbm{1}_{\left\{ [\boldsymbol{v}_{i,t}]_{j} \left(\psi - H\left(\hat{P}^{(t)}(\boldsymbol{w}_{i}), P_{i,j,t}^{IP}\right)\right)\right\}},
    \end{aligned}
\end{equation}
     Since the reward of the proposed method is equivalent to the objective function of the problem (\ref{eq:max2}) and the constraint, the proposed PPO algorithm can solve the minimization problem (\ref{eq:max2}) by maximizing the reward. 
\end{itemize}

\subsection{PPO Algorithm for Total Reward Maximization}
Next, we introduce the entire procedure of training the proposed PPO algorithm for solving the problem (\ref{eq:max2}). 
The proposed PPO algorithm used to update GNN-RNN based model is trained offline, which means that the device connection scheme and model aggregation matrix are trained using the historical data of the DFL system.
In the proposed PPO algorithm, the objective function of the actor is 
\begin{equation}\label{eq:expectedReward}
    \begin{aligned}
        \mathfrak{L}\left(\mathcal{S}_{i,t}, \boldsymbol{\mu}_{i}^{t}, {{\bf{A}}_{t}}\boldsymbol{\theta}_{t}\right)=\mathbb{E}\left[p_{\boldsymbol{\theta}_{t}}A\left(\mathcal{S}_{i,t}, \boldsymbol{\mu}_{i}^{t}, {\bf{A}}_{t}\right)\right],
    \end{aligned}
\end{equation}
where 
\begin{equation}\label{eq:Asa}
    \begin{aligned}
        A\left(\mathcal{S}_{i,t}, \boldsymbol{\mu}_{i}^{t}, {{\bf{A}}_{t}}\right)=r\left(\mathcal{S}_{i,t}, \boldsymbol{\mu}_{i}^{t}, {{\bf{A}}_{t}}\right) + \gamma \boldsymbol{V}_{\boldsymbol{\theta}_{t}^{'}}\left(\mathcal{S}_{i,t}\right) - \boldsymbol{V}_{\boldsymbol{\theta}_{t}^{'}}\left(\mathcal{S}_{i,t}\right),
    \end{aligned}
\end{equation}
and $p_{\boldsymbol{\theta}_{t}}=\boldsymbol{\pi}_{\boldsymbol{\theta}_{t}}/\boldsymbol{\pi}_{{\boldsymbol{\theta}_{t-1}}}$ is the advantage function and the probability ratio of the current policy and the old policy function, respectively. 
To satisfy the trust region constraint in PPO, the proposed PPO-based approach maximizes a clipping surrogate objective function, which is expressed as
\begin{equation}\label{eq:actualReward}
    \begin{aligned}
        &\mathfrak{L}^{\text{c}}\left(\mathcal{S}_{i,t}, \boldsymbol{\mu}_{i}^{t}, {{\bf{A}}_{t}}|\boldsymbol{\theta}_{t}\right) \hfill \\
        &\hspace{0.4cm} = \mathbb{E}\!\left[\min\!\left\{p_{\boldsymbol{\theta}_t}, \! \zeta \left(p_{\boldsymbol{\theta}_{t}}, \!1\!-\!\epsilon ,\!1\!+\!\epsilon\right)\right\}A\left(\mathcal{S}_{i,t}, \boldsymbol{\mu}_{i}^{t}, {{\bf{A}}_{t}}\right)\right],
    \end{aligned}
\end{equation}
where $\epsilon$ is a hyper-parameter that adjusts the clipping fraction of the clipping range, and $\zeta\left(\cdot\right)$ is the clip function. Then, we can update the actor model $\boldsymbol{\theta}_{t}$ by mini-batch SGD method which can be represented as 
\begin{equation}\label{eq:SGDPPO}
    \begin{aligned}
        \boldsymbol{\theta}_t=\boldsymbol{\theta}_{t-1}-\frac{1}{B}\!\!\!\sum\limits_{\left(\mathcal{S}_{i,t},\boldsymbol{\mu}_{i}^{t}, {{\bf{A}}_{t}},r_{t},\mathcal{S}_{i,t+1}\right)} \!\!\! \nabla_{\boldsymbol{\theta}_{t-1}} \mathfrak{L}^{\text{c}}\left(\mathcal{S}_{i,t},\boldsymbol{\mu}_{i}^{t}, {{\bf{A}}_{t}}|\boldsymbol{\theta}_{t-1}\right).
    \end{aligned}
\end{equation}
After conducting $n$ iterations of policy function optimization, we utilize regression on the mean-squared error to adjust the value function $\boldsymbol{V}_{\boldsymbol{\theta}'}$ based on the actual rewards, which is given by
\begin{equation}\label{eq:StateValue}
    \begin{aligned}
        \boldsymbol{\theta}_{t}^{'}=\boldsymbol{\theta}_{t-1}^{'} \! - \! \frac{1}{B} \!\! \sum\limits_{\left(\mathcal{S}_{i,t},\boldsymbol{\mu}_{i}^{t}, {{\bf{A}}_{t}},r_{t},\mathcal{S}_{i,t+1}\right)} \!\!\!\!\!\!\!\!\!\! \nabla_{\boldsymbol{\theta}_{t-1}^{'}}\left(\boldsymbol{V}_{\boldsymbol{\theta}_{t-1}^{'}}\left(\mathcal{S}_{i,t}\right)-r_{i,t}\right)^2.
    \end{aligned}
\end{equation}

By iteratively running the policy updating step (\ref{eq:SGDPPO}) and the state-value updating step (\ref{eq:StateValue}), the parameters $\boldsymbol{\theta}_{t}$ and $\boldsymbol{\theta}_{t}'$ of the policy and state-value can find the relation between the device connections and the total reward, jointly considering the DFL performance, privacy protection, and robust aggregation. The specific training process of the proposed PPO algorithm is summarized in Algorithm \ref{PPOAlgorithm}.


\begin{algorithm}[t]
    \caption{GNN based PPO Algorithm for FL model transmission optimization}
    \label{PPOAlgorithm}
    \begin{algorithmic}[1]
        \STATE \textbf{Initialize} policy parameters $\boldsymbol{\theta}_{t}$, initial state-value function parameters $\boldsymbol{\theta}_{t}^{'}$.
        \FOR{$t=1,2,\ldots,T$}
            \STATE Collect $\left(\mathcal{S}_{i,t},\boldsymbol{\mu}_{i}^{t}, {{\bf{A}}_{t}},r_{t},\mathcal{S}_{i,t+1}\right)$ by running policy $\boldsymbol{\pi}_{\boldsymbol{\theta}_{t}}$ in the environment;
            \STATE Compute advantage estimates, $A\left(\mathcal{S}_{i,t}, \boldsymbol{\mu}_{i}^{t}, {{\bf{A}}_{t}}\right)$ based on the current state-value function $V_{\boldsymbol{\theta}_{t}^{'}}$;
            \STATE Update the policy by maximizing the (\ref{eq:actualReward});
            \STATE Fit state-value function by regression on mean-squared error based on (\ref{eq:StateValue});
        \ENDFOR
    \end{algorithmic}
\end{algorithm}

\section{Simulation Results}
For our simulations, we consider a network with a circular area having a radius $r = 1000$ m, consisting $M$ standard training devices and $B$ Byzantine adversaries that upload random posterior \cite{Trustiness_basedhierarchicaldecentralizedfederatedlearning}.
The other simulation parameters are shown in Tabel \ref{Table:I}. 
We perform simulations using LeNet on MNIST \cite{MNIST}, EMNIST \cite{EMNIST} datasets, and ShuffleNet \cite{shufflenet} on CIFAR-10 \cite{CIFAR10} dataset for the performance comparison.
To simulate data distribution skew across devices, we assume that the proportion of class labels for each device $i$ over $N$ classes is parameterized by a vector $\boldsymbol{\beta}_{i} = \left[\boldsymbol{\beta}_{i,1}, \cdots, \boldsymbol{\beta}_{i, N}\right] \sim{\text{Dir}}_N\left(1\right)$, which follows a Dirichlet distribution \cite{Dirichlet} with classes prior $1$.
The experiments are conducted on a computer with $3.4$GHz IntelCore i$7-13700$KF processor, $64$GB of RAM, GeForce RTX $4090$ GPU running Linux.
\begin{table}[t]
  \centering
  \renewcommand\arraystretch{1.0}
  \caption{Simulation Parameters}
  \label{tab:ExpTable}
  \small
  \setlength{\tabcolsep}{1.2mm}{
      \begin{tabular}{|c|c|c|c|}
          \hline
\textbf{Parameters}&\textbf{Values}&\textbf{Parameters}&\textbf{Values}\\
          \hline
          \emph{W}& 1 MHz & $f$ & $3.3$ GHz\\
          \hline
          $\sigma^2_N$ & -174 dB & $\Gamma$ & $10$ ms\\
          \hline
  \end{tabular}}
  \label{Table:I}
\end{table}

For comparison, we utilize three baselines:
\begin{itemize}
    \item A hierarchical DFL method that assigns devices into $3$ clusters in which a randomly selected device in each cluster acts as the parameter server to collect the model updates encrypted by the CKKS technique and aggregate with the federated average algorithm.
    Then, the parameter servers decrypt and exchange the aggregated model updates with each other for inter-cluster model aggregation, in which the received model with verified performance less than $50\%$ would be removed, as shown in \cite{Trustiness_basedhierarchicaldecentralizedfederatedlearning} (labeled "THDFL" in plots).
    \item A DFL method that enables devices to exchange model updates encrypted by the CKKS technique and a security proof generated by the zk-SNARK method with all available neighbors and aggregate with the federated average algorithm, as shown in \cite{Enhancingprivacypreservation} (labeled "PT-DFL" in plots).
    \item A deterministic-based DFL approach that enables devices to build a random connection scheme and aggregate with the federated average algorithm (labeled "Random connection" in plots).
\end{itemize}


\subsection{DFL Performance Analysis}

\begin{figure}[t]
    \centering
    \includegraphics[width=7.5cm]{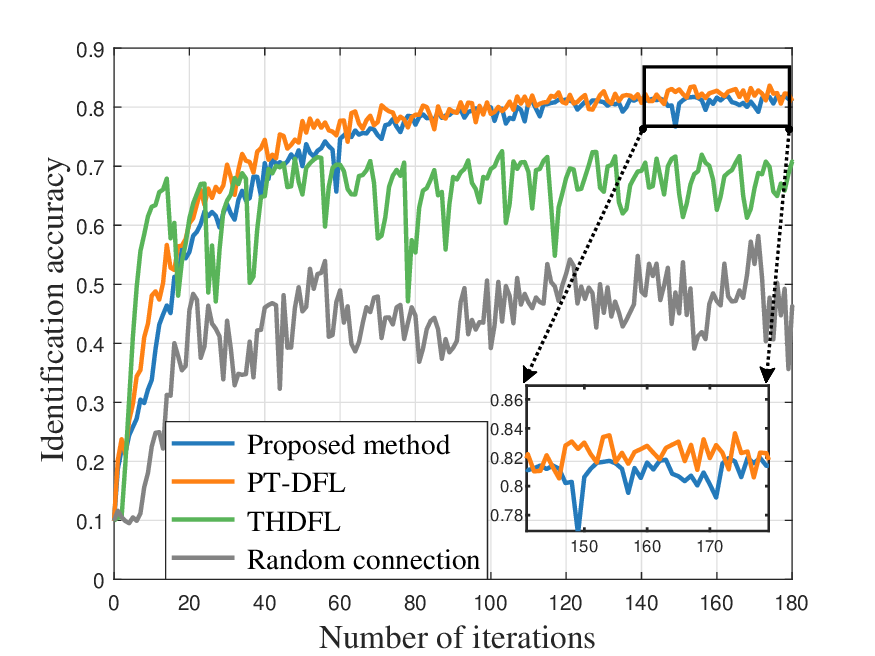}
    \label{1}
    \caption{Identification accuracy vs. number of iterations.}
    \label{fig_acc1}
\end{figure}
Fig. \ref{fig_acc1} shows how the average identification accuracy of all considered algorithms changes as the number of iterations varies on the MNIST dataset for $R=6$ and $B=1$. 
From Fig. \ref{fig_acc1}, we can see that the proposed method achieves similar accuracy performance with the PT-DFL method.
This is due to the fact that the proposed method can detect Byzantine adversaries as accurately as PT-DFL which is based on the zk-SNARK method.
From Fig. \ref{fig_acc1}, we can also see that the proposed method achieves up to $13\%$ and $30\%$ improvement compared to the THDFL method and random connection scheme, respectively
The $13\%$ and $30\%$ gain stem from the fact that the proposed method enables each device to distinguish the Byzantine adversaries and reduce their model aggregation weight, whereas the THDFL method and the random connection method permit poisoned model updates to participate in the federated averaging aggregation process, thereby degrading identification accuracy.
Fig. \ref{fig_acc1} shows that the proposed method achieves stable convergence compared with the THDFL method and the random connection method.
This implies that our proposed method utilizes a dynamic device connection that minimizes $\xi_t$ to achieve yields significant gains in convergence speed compared to the THDFL method and the random connection method, which adopt a completely random clustering and connection strategy.

\begin{figure}[t]
  \centering
  \includegraphics[width=7.5cm]{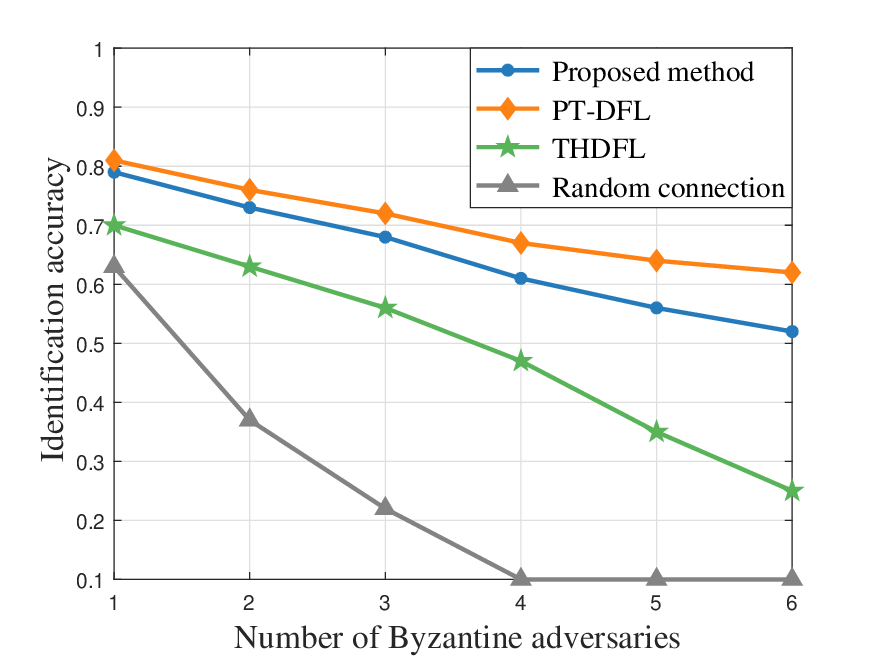}
  \centering
  \caption{Identification accuracy vs. the number of Byzantine adversaries.}
  \label{fig_acc3}
\end{figure}

Fig. \ref{fig_acc3} shows how the average DFL identification accuracy changes as the number of Byzantine adversaries varies on the MNIST dataset for $R=12$ case. 
From Fig. \ref{fig_acc3}, we can see that as the proportion of Byzantine adversaries increases, all considered methods have lower average identification accuracy.
This is due to the fact that the number of standard training devices decreased and the destructiveness of the Byzantine adversaries increased.
From Fig. \ref{fig_acc3}, we can also observe that the PT-DFL and the proposed method can improve the identification accuracy by up to $35\%$ and $25\%$ compared with THDFL and random connection scheme, respectively.
The $35\%$ gain stems from the fact the PT-DFL method can detect and exclude all malicious devices by zk-SNARK based method.
The $25\%$ gain stems from the fact that the proposed method can detect the Byzantine adversaries based on partial information (e.g., received posterior, approximated prior, and device connection) and decrease their model aggregation weight thus ensuring Byzantine robustness. 
Fig. \ref{fig_acc3} also shows that the accuracy of THDFL and the random connection scheme collapses when the proportion of Byzantine adversaries reaches $50\%$.
This is due to the fact that the THDFL and the random connection scheme enable all Byzantine adversaries to participate in the model aggregation with the same aggregation weight as the standard training devices, and hence they are impacted by Byzantine adversaries directly and eventually crash.


\begin{figure}[t]
    \centering
    \begin{minipage}{0.48\linewidth}
        \centering
        \includegraphics[width=\linewidth]{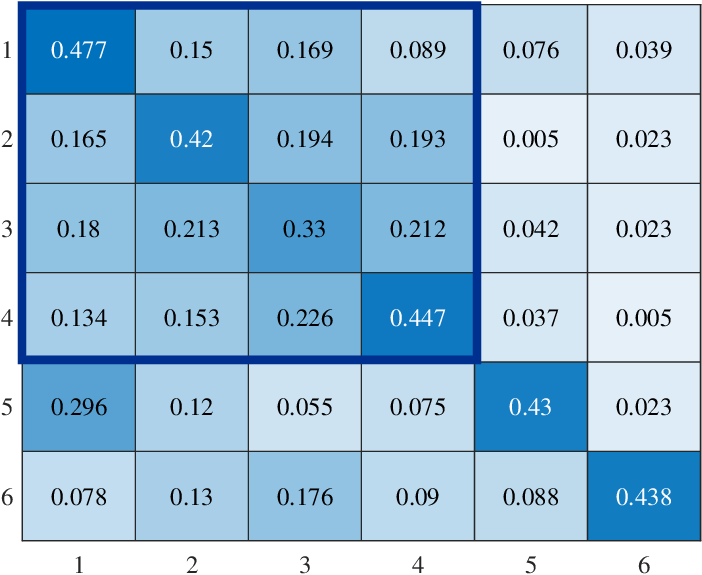}
        \vspace{0.1cm} 
        {\small (a) $R=6$ with $B=2$}
    \end{minipage}
    \hfill
    \begin{minipage}{0.48\linewidth}
        \centering
        \includegraphics[width=\linewidth]{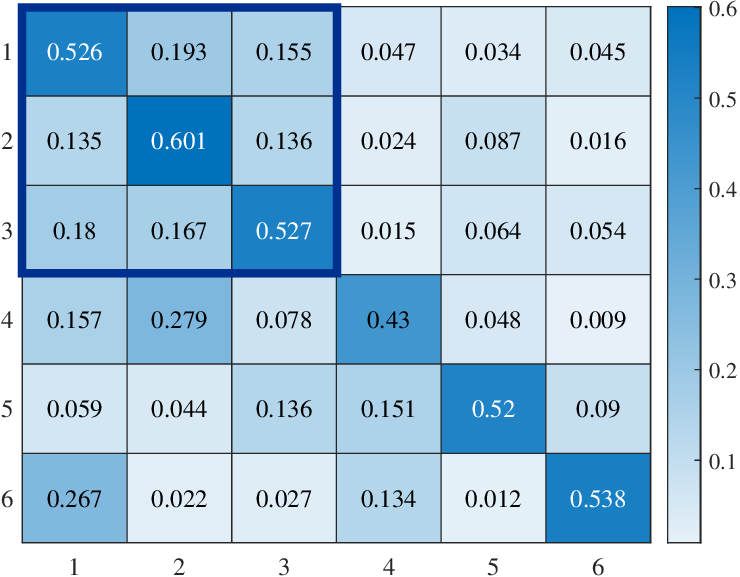}
        \vspace{0.1cm}
        {\small (b) $R=6$ with $B=3$.}
    \end{minipage}
    
    \caption{The proposed method defends results against Byzantine adversaries. The model aggregation weights across standard training devices are surrounded by blue boxes.}
    \label{heatmap} 
\end{figure}
Fig. \ref{heatmap} presents a visualization of the averaged model aggregation weight matrix ${\bf{A}}_t\in\left[0,1\right]^{6\times 6}$ generated by the proposed method trained in $R=6$ and $B=2$ case.
Fig. \ref{heatmap} (a) and Fig. \ref{heatmap} (b) are in $B=2$ and $B=3$ cases, respectively.
From Fig. \ref{heatmap} we can see that the model aggregation weight corresponding to the Byzantine adversaries exhibits significant deviations compared to those of the standard training devices (surrounded by blue boxes).
This is due to the fact that the proposed method can decrease the model aggregation weight of Byzantine adversaries, thus ensuring Byzantine robustness.
From Fig. \ref{heatmap} we can also see that the model aggregation weights of each device's local model updates increase as the proportion of Byzantine adversaries increases.
This is due to the fact that the Byzantine robustness of the proposed method relies on distinguishing the model updates that do not reflect the approximated prior which is considered as an averaged aggregation of all received model updates.
Consequently, an increasing proportion of Byzantine adversaries also interferes with the credibility of other standard training devices and increases the model aggregation weights of each device's local model updates.
\subsection{Sensitivity Analysis}
\begin{figure}[t]
  \centering
  \includegraphics[width=7.5cm]{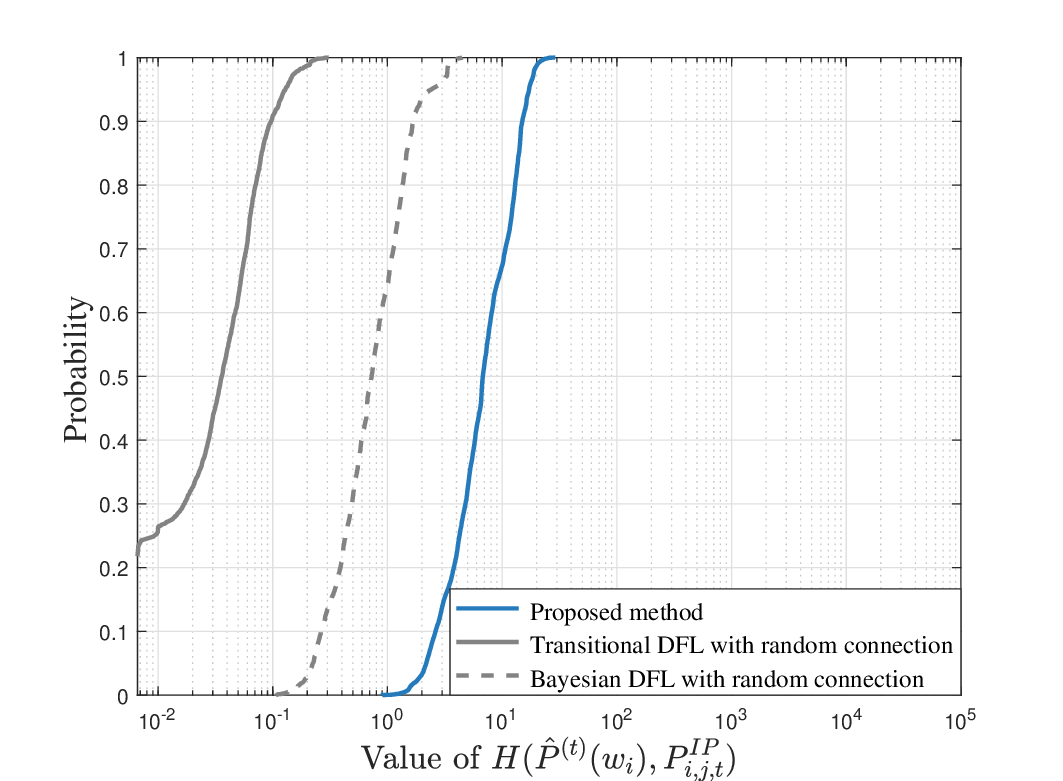}
  \centering
  \caption{Cumulative Distribution Function (CDF) of $H\left(\hat{P}^{\left(t\right)}\left(\boldsymbol{w}_{i}\right), P_{j,i,t}^{IP}\right)$}
  \vspace{-0.1cm}
  \label{privacy1}
\end{figure}

Fig. \ref{privacy1} assesses the gap between the local prior and the approximated prior for the considered methods without encryption with the EMNIST dataset in $R=12$ case.
The considered baselines consist of the proposed method, the traditional DFL (deterministic based) with a random connection scheme, and the Bayesian DFL with a random connection scheme.
The gap considered in the traditional DFL is measured based on the gap between the local model weight and the approximated model weight by other devices.
The results are shown as the cumulative distribution function (CDF) of the value of the gap obtained over training rounds.
We do not consider encryption-based methods (e.g., PT-DFL and THDFL) as their model updates are exchanged in ciphertext, which requires access to decryption keys.
From Fig. \ref{privacy1} we can see that the proposed method yields significant gains in the gap between the local model and the approximated model compared to other baselines.
This is because the proposed method keeps the local prior invisible to neighbors by utilizing Bayesian training method and optimizes the gap between the true prior and the approximated prior by adjusting the device connection scheme based on the proposed GNN method.
Besides, although the Bayesian based DFL with random connection method can protect the local posterior through model aggregation in a Bayesian fashion compared to the traditional DFL, it ignores the optimization of neighbor selection for model exchange, which also causes privacy leakage.
From Fig. \ref{privacy1} we can also see that traditional DFL with random connection has about $30\%$ of complete model exposure (where the gap between the local model weight and the approximated model weight equals zero), which implies its high risks of privacy leakage.
This is due to the fact that the adversaries in the traditional DFL with random connection can fully reconstruct their neighbors' updated local models easily when they receive all model updates transmitted during the same communication round.
From Fig. \ref{privacy1}, we can demonstrate that the proposed Bayesian DFL and GNN based method can significantly enhance privacy preservation compared with the traditional methods.

\begin{figure}[t]
\centering
    \subfigure[EMNIS./Sample.]{
    \begin{minipage}[t]{0.27\linewidth}
    \centering
    \includegraphics[width=1in]{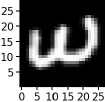}
    \end{minipage}%
    }\quad
    \subfigure[EMNIS./Random.]{
    \begin{minipage}[t]{0.27\linewidth}
    \centering
    \includegraphics[width=1in]{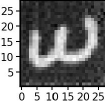}
    \end{minipage}%
    }\quad 
    \subfigure[EMNIS./Proposed.]{
    \begin{minipage}[t]{0.27\linewidth}
    \centering
    \includegraphics[width=1in]{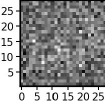}
    \end{minipage}
    } \\
    \subfigure[MNIST/Sample.]{
    \begin{minipage}[t]{0.27\linewidth}
    \centering
    \includegraphics[width=1in]{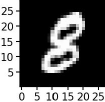}
    \end{minipage}%
    }\quad
    \subfigure[MNIST/Random.]{
    \begin{minipage}[t]{0.27\linewidth}
    \centering
    \includegraphics[width=1in]{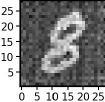}
    \end{minipage}%
    }\quad
    \subfigure[MNIST/Proposed.]{
    \begin{minipage}[t]{0.27\linewidth}
    \centering
    \includegraphics[width=1in]{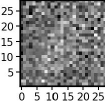}
    \end{minipage}
    } \\
    \subfigure[CIFAR./Sample.]{
    \begin{minipage}[t]{0.27\linewidth}
    \centering
    \includegraphics[width=1in]{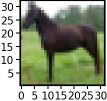}
    \end{minipage}%
    }\quad
    \subfigure[CIFAR./Random.]{
    \begin{minipage}[t]{0.27\linewidth}
    \centering
    \includegraphics[width=1in]{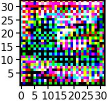}
    \end{minipage}%
    }\quad
    \subfigure[CIFAR./Proposed.]{
    \begin{minipage}[t]{0.27\linewidth}
    \centering
    \includegraphics[width=1in]{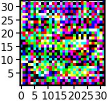}
    \end{minipage}
    }
\centering
\caption{Defending results against DLG attacks.}
\label{privacy2}
\end{figure}
\subsection{Privacy Analysis}
Fig. \ref{privacy2} is a visualization of reconstruction results at the average $H\left(\hat{P}^{\left(t\right)}\left(\boldsymbol{w}_{i}\right), P_{i,j,t}^{IP}\right)$ of the random connection scheme and the proposed method for EMNIST, MNIST, and the CIFAR-10 dataset corresponding to the CDF plot in Fig. \ref{privacy1}.
From Fig. \ref{privacy2}, we can see that the proposed algorithm effectively prevents the leakage of private data compared to the random connection scheme.
This is due to the fact that the proposed method keeps the local prior invisible and prevents other devices from accurate approximations of local prior through active device connection optimization.

\subsection{Complexity Analysis}
\begin{table}[!t]
    \centering
    \begin{threeparttable}
        \caption{Delay of each operation}
        \label{Tab:Overhead}
        \begin{tabularx}{\linewidth}{|c|X|c|}
            \hline
            \textbf{Index} & \textbf{Operation} & \textbf{Delay} \\ 
            \hline
            1 & CKSS encryption \cite{CKKS, Tenseal} & $11.4 \ \mu$s/byte \\
            \hline
            2 & CKSS addition & $9.54 \ \mu$s/byte \\
            \hline
            3 & CKSS decryption & $3.82 \ \mu$s/byte \\
            \hline
            4 & zk-SNARK proof generation\tnote{1} & $7$ s \\
            \hline
            5 & zk-SNARK proof verify  & $35\ m$s \\
            \hline
            6 & Proposed method  & $30\ m$s \\
            \hline
        \end{tabularx}
        \begin{tablenotes}
            \footnotesize
            \item[1] Containing delay of setup and the proof generation with input depth $500$ and kernel size $5$ \cite{Enhancingprivacypreservation, zkBench}.
        \end{tablenotes}
    \end{threeparttable}
\end{table}

In Table~\ref{Tab:Overhead} we provide a computation delay comparison associated with the CKSS homomorphic encryption algorithm employed in both the PT-DFL and THDFL methods, the zk-SNARK algorithm utilized in the PT-DFL method, and our proposed method.
From Table~\ref{Tab:Overhead}, we can see that the overhead of the CKSS homomorphic encryption operation increases linearly with the size of the model updates increases, and the zk-SNARK proof generation operation to validate the security is unbearable compared to other considered methods.
This is due to the fact that the proposed method satisfies privacy and security requirements only by controlling the information difference (the accuracy of approximated prior) between devices without cryptographic operations.
Moreover, the devices selected as parameter servers in the THDFL scheme need to perform CKKS encryption, addition and decryption, while other devices only need to perform CKSS encryption, which causes unfairness.
\begin{figure}
    \centering
    \includegraphics[width=7.5cm]{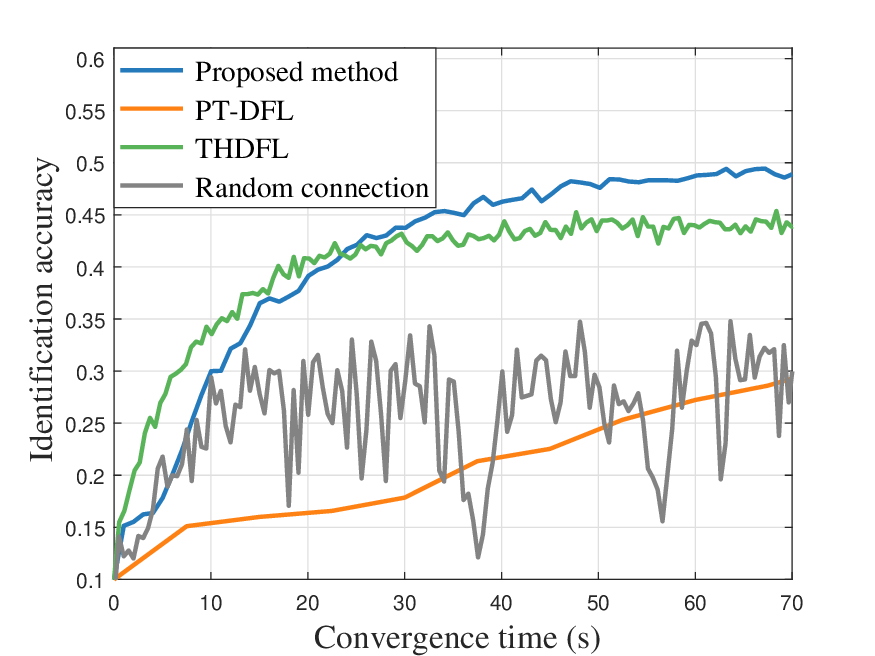}
    \label{1}
    \caption{Identification accuracy vs. convergence time.}
    \label{fig_acc2}
\end{figure}

Fig. \ref{fig_acc2} shows how the average DFL identification accuracy changes as the convergence time varies with CIFAR-10 dataset in $R=6$ and $B=1$ case. 
In this figure, we can see that, the proposed method increase the identification accuracy by up to $5\%$ and $20\%$ compared to THDFL method and PT-DFL method, respectively.
This is due to the fact that PT-DFL and THDFL suffer from a heavy computational delay introduced by encryption and proof generation operations. 
In contrast, the proposed method allows devices to exchange model updates without cryptographic operations, enabling more training rounds within the same time consumption.

\section{Conclusion}
In this paper, we proposed a novel DFL framework in which mobile devices collaboratively train an ML model via exchanging their model updates with an appropriate subset of neighbors (i.e., device connection scheme) to achieve Byzantine robustness, privacy preserving, and convergence acceleration. 
We first analytically characterized the relationship between the dynamic device connection and Byzantine adversary's detection, privacy protection requirement, and DFL convergence acceleration.
Through mathematical analysis, we found that: 1) hiding critical information from neighbors can effectively prevent private data from leakage; 2) devices can verify whether the models received from neighbors have been updated as expected to identify Byzantine adversaries; 3) the convergence speed of DFL can be modeled as the characteristics of the device connection matrix which are affected by device connection scheme design.
Then, to jointly achieve these goals by using partial information (e.g., device connection and model updates) in the DFL framework, we proposed a GNN-based RL method that allows each device to independently infer the device connections and model aggregation matrix.
Numerical evaluation on real-world machine learning tasks demonstrates that the proposed method yields significant gains in Byzantine robustness, privacy-preserving, and convergence acceleration with lightweight overhead compared to conventional approaches.

For future work, we are going to consider a more general case that mobile devices may depart dynamically from DFL.
In this case, the devices' mobility patterns will play an important role in the problem formulation.
Moreover, in this case, DFL should enable each departed device to selectively eliminate identifiable information through machine unlearning which would be very interesting and technically challenging due to the distributed characteristic of DFL.

\appendix

\subsection{Proof of Theorem \ref{thm:theorem1}}
\begin{proof}\label{proof:Theorem1}
To prove Theorem~\ref{thm:theorem1}, we first introduce data likelihood distribution matrix $P^{\left(t\right)}_{\text{l}}\left(\mathcal{D}|\boldsymbol{w}\right)$,
prior matrix $\hat{P}^{\left(t\right)}\left(\boldsymbol{w}\right)$, 
posterior matrix $P^{\left(t\right)}\left({\boldsymbol{w}|\mathcal{D}}\right)$, and evidence matrix $P_{\mathcal{D}}$ as follows:
\begin{equation}
    \begin{aligned}
    & P^{\left(t\right)}_{\text{l}}\left(\mathcal{D}|\boldsymbol{w}\right) = \left[P_{\text{l}}\left(\mathcal{D}_{1}|\boldsymbol{w}_{1}\right),\cdots,P_{\text{l}}\left(\mathcal{D}_{M}|\boldsymbol{w}_{M}\right)\right], \\
    & \hat{P}^{\left(t\right)}\left(\boldsymbol{w}\right) = \left[\hat{P}^{\left(t\right)}\left(\boldsymbol{w}_{1}\right),\cdots,\hat{P}^{\left(t\right)}\left(\boldsymbol{w}_{M}\right)\right], \\
    & P^{\left(t\right)}\left({\boldsymbol{w}|\mathcal{D}}\right)=\left[P\left({\boldsymbol{w}_{1}|\mathcal{D}}_{1}\right),\cdots,P\left({\boldsymbol{w}_{M}|\mathcal{D}}_{M}\right)\right], \\
    & P_{\mathcal{D}}=\left[P\left(\mathcal{D}_{1}\right),\cdots,P\left(\mathcal{D}_{M}\right)\right].
    \end{aligned}
\end{equation}
The local model updates process (\ref{eq:BayesFunction}) can be rewritten as
\begin{equation}\label{rewr:Update}
    \begin{aligned}
         P^{\left(t\right)}\left(\boldsymbol{w}| {\mathcal{D}}\right)=\frac{P_{\text{l}}^{\left(t\right)} \left(\mathcal{D}|\boldsymbol{w}\right)\hat{P}^{\left(t\right)}\left(\boldsymbol{w}\right)}{P\left(\mathcal{D}\right)}.
    \end{aligned}
\end{equation}
The model aggregation process (\ref{eq:CFLUpdate}) can be rewritten as 
\begin{equation}
    \begin{aligned}
        \log\hat{P}^{\left(t\right)}\left(\boldsymbol{w}\right) = \log P^{\left(t-1\right)}\left(\boldsymbol{w}|\mathcal{D}\right){\bf{A}}_{t-1}.
    \end{aligned}
\end{equation}
Moreover, the matrix of the average aggregation of all devices' posterior $P^{\star\left(t\right)}\left(\boldsymbol{w}\right)$ is given by 
\begin{equation}
    \begin{aligned}
    \log P^{\star\left(t\right)}\left(\boldsymbol{w}\right) =  \log \left(P^{\left(t\right)}\left({\boldsymbol{w}|\mathcal{D}}\right)\right) {\bf{J}}.
    \end{aligned}
\end{equation}
where ${\bf{J}}=\left\{\frac{1}{M}\right\}^{M \times M}$.
The gap between $\log  P^{\left(t\right)}\left(\boldsymbol{w}|\mathcal{D}\right)$ and $\log P^{\star\left(t\right)}\left(\boldsymbol{w}\right)$ can be given by
\begin{equation}\label{Recursion1}
    \begin{aligned}
    & E^{\left(t\right)}\left(\boldsymbol{w}\right) 
    = \log P^{\left(t\right)}\left(\boldsymbol{w}| {\mathcal{D}}\right) - \log P^{\star\left(t\right)}\left(\boldsymbol{w}\right) \\
    & = \log\left(\frac{P^{\left(t-1\right)}_{\text{l}}\left(\mathcal{D}|\boldsymbol{w}\right)}{P^{\left(t-1\right)}_{\mathcal{D}}}\right){\bf{A}}_{t-1}^{\top} + \log \hat{P}^{\left(t-1\right)}\left(\boldsymbol{w}\right) {\bf{A}}_{t-1}^{\top} \\
    & - \left(\log\left(\frac{P^{\left(t-1\right)}_{\text{l}}\left(\mathcal{D}|\boldsymbol{w}\right)}{P^{\left(t-1\right)}_{\mathcal{D}}}\right){\bf{A}}_{t-1}^{\top} + \log \hat{P}^{\left(t-1\right)}\left(\boldsymbol{w}\right) {\bf{A}}_{t-1}^{\top}\right){\bf{J}} \\
    & = \left(\log P^{\left(t-1\right)}\left(\boldsymbol{w}|\mathcal{D}\right) - \log P^{\star\left(t-1\right)}\left(\boldsymbol{w}\right) \right){\bf{A}}_{t-1}^{\top} \\
    & + \log\left(\frac{P^{\left(t-1\right)}_{\text{l}}\left(\mathcal{D}|\boldsymbol{w}\right)}{P^{\left(t-1\right)}_{\mathcal{D}}}\right){\bf{A}}_{t-1}^{\top}\left({\bf{I}}-{\bf{J}}\right)
     \\
    & = E^{\left(t-1\right)}\left(\boldsymbol{w}\right) {\bf{A}}_{t-1}^{\top} 
    + \log\left(\frac{P^{\left(t-1\right)}_{\text{l}}}{P^{\left(t-1\right)}_{\mathcal{D}}}\right){\bf{A}}_{t-1}^{\top}\left({\bf{I}}-{\bf{J}}\right),
    \end{aligned}
\end{equation}
where the first equality holds due to (\ref{rewr:Update}).
Therefore, we can rewrite $E^{\left(t\right)}\left(\boldsymbol{w}\right)$ in (\ref{Recursion1}) as 
\begin{equation}
    \begin{aligned}
    & E^{\left(t\right)}\left(\boldsymbol{w}\right) \\ 
    & = 
    E^{\left(t-1\right)}\left(\boldsymbol{w}\right) {\bf{A}}_{t-1}^{\top} 
    + \log\left(\frac{P^{\left(t-1\right)}_{\text{l}}\left(\mathcal{D}|\boldsymbol{w}\right)}{P^{\left(t-1\right)}_{\mathcal{D}}}\right){\bf{A}}_{t-1}^{\top}\left({\bf{I}}-{\bf{J}}\right)\\
    & = E^{\left(t-2\right)}\left(\boldsymbol{w}\right) {\bf{A}}_{t-2}^{\top}{\bf{A}}_{t-1}^{\top} \!+\! \log\left(\frac{P^{\left(t-1\right)}_{\text{l}}\left(\mathcal{D}|\boldsymbol{w}\right)}{P^{\left(t-1\right)}_{\mathcal{D}}}\right){\bf{A}}_{t-1}^{\top}\left({\bf{I}}\!-\!{\bf{J}}\right) \\
    & \hspace{0.4cm} + \log\left(\frac{P^{\left(t-2\right)}_{\text{l}}\left(\mathcal{D}|\boldsymbol{w}\right)}{P^{\left(t-2\right)}_{\mathcal{D}}}\right){\bf{A}}_{t-2}^{\top}\left({\bf{I}}-{\bf{J}}\right){\bf{A}}_{t-1}^{\top}  \\
    & = E^{\left(1\right)}\left(\boldsymbol{w}\right) \prod\limits_{\tau=1}^{t-1}{\bf{A}}_{\tau}\!\! + \!\!\sum\limits_{\tau=1}^{t-1}\log\!\left(\frac{P^{\left(\tau\right)}_{\text{l}}\left(\mathcal{D}|\boldsymbol{w}\right)}{P^{\left(\tau\right)}_{\mathcal{D}}}\right){\bf{A}}_{\tau}^{\top}\left({\bf{I}}\!-\!{\bf{J}}\right)\!\!\!\!\prod\limits_{z=\tau+1}^{t-1}\!\!{\bf{A}}_{z}^{\top} \\
    & = \sum\limits_{\tau=1}^{t-1}\log\left(\frac{P^{\left(\tau\right)}_{\text{l}}\left(\mathcal{D}|\boldsymbol{w}\right)}{P^{\left(\tau\right)}_{\mathcal{D}}}\right){\bf{A}}_{\tau}^{\top}\left({\bf{I}}-{\bf{J}}\right)\prod\limits_{z=\tau+1}^{t-1}{\bf{A}}_{z}^{\top}.
    \end{aligned}
\end{equation}
The last equality is due to the assumption that all local models have the same initial parameters, thus, $E^{\left(1\right)}=0$.
Next, we derive the Frobenius norm of $E^{\left(t\right)}\left(\boldsymbol{w}\right)$ to measure the upper bound of the gap between the average posterior $P^{\star\left(t\right)}\left(\boldsymbol{w}| {\mathcal{D}}\right)$ and each local posterior $P^{\left(t\right)}\left(\boldsymbol{w}| {\mathcal{D}}\right)$ which is given by
\begin{equation}\label{the1:inequality}
    \begin{aligned}        & ||E^{\left(t+1\right)}\left(\boldsymbol{w}\right)||_{\rm{F}} \\ &\leqslant ||\sum\limits_{\tau=1}^{t}\log\left(\frac{P^{\left(\tau\right)}_{\text{l}}\left(\mathcal{D}|\boldsymbol{w}\right)}{P^{\left(\tau\right)}_{\mathcal{D}}}\right){\bf{A}}_{\tau}^{\top}\left({\bf{I}}-{\bf{J}}\right)\prod\limits_{z=\tau+1}^{t-1}{\bf{A}}_{z}^{\top}||_{\rm{F}} \\
    & \leqslant \sum\limits_{\tau=1}^{t}||{\bf{A}}_{\tau}^{\top}\left({\bf{I}}-{\bf{J}}\right)\prod\limits_{z=\tau+1}^{t}{\bf{A}}_{z}^{\top}||_{\rm{op}}||\log\left(\frac{P^{\left(\tau\right)}_{\text{l}}}{P^{\left(\tau\right)}_{\mathcal{D}}}\right)||_{\rm{F}} \\
    & \leqslant L\sqrt{M} \sum\limits_{\tau=1}^{t}||{\bf{A}}_{\tau}^{\top}\left({\bf{I}}-{\bf{J}}\right)\prod\limits_{z=\tau+1}^{t}{\bf{A}}_{z}^{\top}||_{\rm{op}}.
    \end{aligned}
\end{equation}
The second and the third inequalities are due to [Lemma 5,  \cite{DecentralizedFederatedLearningBalancingCommunicationandComputingCosts}] and Assumption 2, respectively.
Then, we need to analyze the upper bound of $||{\bf{A}}_{\tau}^{\top}\left({\bf{I}}-{\bf{J}}\right)\prod\limits_{z=\tau+1}^{t}{\bf{A}}_{z}^{\top}||_{\rm{op}}$ according to the definition of the matrix operator norm, which is given by
\begin{equation}\label{max_lambda}
    \begin{aligned} 
        & \!\!||{\bf{A}}_{\tau}^{\top}\left({\bf{I}}-{\bf{J}}\right)\prod\limits_{z=\tau+1}^{t}{\bf{A}}_{z}^{\top}||_{\rm{op}} \\ 
        & = \!\!\sqrt{\lambda_{\text{max}}\!\!\left(\!\left({\bf{A}}_{\tau}^{\top}\left({\bf{I}}-{\bf{J}}\right)\!\!\!\prod\limits_{z=\tau+1}^{t}{\bf{A}}_{z}^{\top}\right)^{\top}\!\!\!\left({\bf{A}}_{\tau}^{\top}\left({\bf{I}}-{\bf{J}}\right)\!\!\!\prod\limits_{z=\tau+1}^{t}{\bf{A}}_{z}^{\top}\right)\!\right)} \\
        & \overset{\bigtriangleup}{=} \!\! \sqrt{\lambda_{\text{max}}\!\!\left(\left({\bf{A}}_{\tau}^{\top}\left({\bf{I}}-{\bf{J}}\right){\bf{Q}}_{\tau+1}\right)\left({\bf{A}}_{\tau}^{\top}\left({\bf{I}}-{\bf{J}}\right){\bf{Q}}_{\tau+1}\right)^{\top}\right)} \\
        & = \!\! \sqrt{\lambda_{\text{max}}\!\!\left(\left({\bf{A}}_{\tau}^{\top}\left({\bf{I}}-{\bf{J}}\right){\bf{Q}}_{\tau+1}\right)\left(\left({\bf{Q}}_{\tau+1}\right)^{\top}\left({\bf{I}}-{\bf{J}}\right)^{\top}{\bf{A}}_{\tau}\right)\right)} \\
        & = \!\! \sqrt{\lambda_{\text{max}}\!\!\left(\left({\bf{I}}-{\bf{J}}\right){\bf{A}}_{\tau}^{\top}\left({\bf{Q}}_{\tau+1}\left({\bf{Q}}_{\tau+1}\right)^{\top}\right){\bf{A}}_{\tau}\left({\bf{I}}-{\bf{J}}\right)\right)} \\
        & = \!\! \sqrt{\lambda_{\text{max}}\!\!\left(\left({\bf{I}}-{\bf{J}}\right){\bf{Q}}_{\tau}{\bf{Q}}_{\tau}^{\top}\left({\bf{I}}-{\bf{J}}\right)\right)} = \sqrt{\lambda_{\text{max}}\!\!\left(\left({\bf{I}}-{\bf{J}}\right){\bf{Q}}_{\tau}{\bf{Q}}_{\tau}^{\top}\right)},
    \end{aligned}
\end{equation}
where ${\bf{Q}}_{\tau}=\prod\limits_{z=\tau}^{t}{\bf{A}}_{z}^{\top}$ with ${\bf{Q}}_{\tau}{\bf{Q}}_{\tau}^{\top}$ being a real symmetric matrix since every ${\bf{A}}_{z}$ is row random matrix. The term $\lambda_{\text{max}}\left({{\bf{C}}}\right)$ represent the the maximum eigenvector of matrix ${{\bf{C}}}$.
Then, we assume ${\bf{Q}}_{\tau}=\prod\limits_{z=\tau}^{t}{\bf{A}}_{z}^{\top}$ can be decomposed into orthogonal matrix of the eigenvectors ${\bf{P}}_{\tau}^{\top}$ and diagonal matrix of eigenvalues ${\bf{\Lambda}}_{\tau}$ which is sort in descending order.
Then, we have ${\bf{Q}}_{\tau}{\bf{Q}}_{\tau}^{\top}={\bf{P}}_{\tau}^{\top}{\bf{\Lambda}}_{\tau}{\bf{P}}_{\tau}$.
Thus, (\ref{max_lambda}) can be rewritten by
\begin{equation}
    \begin{aligned}
        \sqrt{\lambda_{\text{max}}\left(\left({\bf{I}}-{\bf{J}}\right){\bf{Q}}_{\tau}{\bf{Q}}_{\tau}^{\top}\right)} \!\! = \!\! \sqrt{\!\text{max}\!\!\left(\!{\bf{\Lambda}}_{\tau}\!\!-\!\! \diag\left\{1,0,\cdots\!,0\right\}\right)} \!\!\overset{\bigtriangleup}{=}\!\!\sqrt{\xi_{\tau}},
    \end{aligned}
\end{equation}
where $\xi_{\tau} \in\left[0,1\right]$ is the second largest eigenvalue of $\prod\limits_{z=\tau}^{t}{\bf{A}}_{z}^{\top}\left(\prod\limits_{z=\tau}^{t}{\bf{A}}_{z}^{\top}\right)^{\top}$.
The first equality is due to the fact that every row of ${\bf{J}}{\bf{Q}}_{\tau}{\bf{Q}}_{\tau}^{\top}$ is identical, hence, the rank of the matrix is $1$. 
Consequently, ${\bf{J}}{\bf{Q}}_{\tau}{\bf{Q}}_{\tau}^{\top}$ has precisely one eigenvalue equal to $1$, with all other eigenvalues being $0$.
Then, (\ref{the1:inequality}) can be rewritten as
\begin{equation}
    \begin{aligned}
    \frac{1}{M}\sum\limits_{i=1}^{M}||\log P^{\left(t+1\right)}\left(\boldsymbol{w}_{i}|\mathcal{D}_{i}\right) \!-\!&\log P^{\star\left(t+1\right)}\left(\boldsymbol{w}\right)||^{2} \\
        & \leq 
        \left(L\sqrt{M} \sum\limits_{\tau=1}^{t} \sqrt{\xi_{\tau}}\right)^{2}.
    \end{aligned}
\end{equation}
This completes the proof.
\end{proof}

\end{document}